\documentclass[11pt]{article}
\usepackage[a4paper,top=3cm,bottom=2cm,left=2.8cm,right=2.8cm,marginparwidth=1.75cm]{geometry}
\usepackage[round]{natbib}

\usepackage[normalem]{ulem}
\usepackage[none]{hyphenat}
\usepackage{breakcites}
\usepackage[utf8]{inputenc} %
\usepackage[T1]{fontenc}    %
\usepackage[colorlinks=True,citecolor=blue,urlcolor=blue,pagebackref=true,backref=true,linkcolor=blue]{hyperref}
\renewcommand*{\backrefalt}[4]{%
	\ifcase #1 \footnotesize{(Not cited.)}%
	\or        \footnotesize{(Cited on page~#2.)}%
	\else      \footnotesize{(Cited on pages~#2.)}%
	\fi}
\usepackage{url}            %
\usepackage{booktabs}       %
\usepackage{amsfonts}       %
\usepackage{nicefrac}       %
\usepackage{microtype}      %
\usepackage{xcolor}         %

\usepackage{float}
\usepackage{amsthm}
\usepackage{mathtools}
\usepackage{soul}
\mathtoolsset{showonlyrefs}
\usepackage{bm}
\usepackage{enumitem}
\usepackage{multirow}
\usepackage{nameref}
\usepackage{wrapfig}
\usepackage{scalerel}
\allowdisplaybreaks
\usepackage{dsfont}
\usepackage[noend]{algpseudocode}
\usepackage{xcolor}
\usepackage{thmtools,thm-restate}
\usepackage[linesnumbered,ruled,vlined]{algorithm2e}

\SetCommentSty{mycommfont}
\SetKwInput{KwInput}{Input}                %
\SetKwInput{KwOutput}{Output}  
\SetKwInput{KwInitialization}{Initialization}

\newtheorem{remark}{Remark}

\newtheorem{definition}{Definition}

\newtheorem{assumption}{Assumption}

\newtheorem{corollary}{Corollary}
\newtheorem{lemma}{Lemma}
\newtheorem{proposition}{Proposition}

\DeclareMathAlphabet{\mathbsf}{OT1}{cmss}{bx}{n}%
\DeclareMathAlphabet{\mathssf}{OT1}{cmss}{m}{sl}%

\DeclareMathOperator*{\argmin}{arg\,min}

\DeclarePairedDelimiterX{\infdivx}[2]{(}{)}{%
	#1\;\delimsize\|\;#2%
}

\newcommand{\dimOne}{k_1}	%
\newcommand{\dimTwo}{k_2}	%
\newcommand{\rvv}{{\mathssf{v}}}	%
\newcommand{\rvx}{{\mathssf{x}}}	%
\newcommand{\rvz}{{\mathssf{z}}}	%
\newcommand{\rvbv}{{\mathbsf{v}}} %
\newcommand{\rvbx}{{\mathbsf{x}}} %
\newcommand{\svbv}{{\mathbf{v}}} %
\newcommand{\svbx}{{\mathbf{x}}} %
\newcommand{\svby}{{\mathbf{y}}} %
\newcommand{\bM}{{\mathbf{M}}} %
\newcommand{\bN}{{\mathbf{N}}} %
\newcommand{\bU}{{\mathbf{U}}} %
\newcommand{\bH}{{\mathbf{H}}} %
\newcommand{\hbH}{\hat{\mathbf{H}}} %

\newcommand{\br}{{r}}
\newcommand{\bd}{{d}}

\newcommand{\Reals}{\mathbb{R}} %

\newcommand{\bphi}{\boldsymbol{\phi}} 
\newcommand{\btheta}{\boldsymbol{\theta}} 
\newcommand{\ThetaStar}{{\Theta}^*} %

\newcommand{\hTheta}{\hat{\Theta}} %
\newcommand{\hThetan}{\hat{\Theta}_n} %
\newcommand{\hThetaEps}{\hat{\Theta}_{\epsilon,n}}
\newcommand{\DensityExpfam}{f_{\rvbx}( \svbx;\btheta )} %
\newcommand{\DensityX}{f_{\rvbx}( \svbx;\Theta )} %
\newcommand{\DensityXfun}{f_{\rvbx}( \cdot;\Theta )} %
\newcommand{\DensityXTrue}{f_{\rvbx}( \svbx;\ThetaStar )} %
\newcommand{\DensityXTruefun}{f_{\rvbx}( \cdot;\ThetaStar )} %
\newcommand{\cX}{\mathcal{X}} %
\newcommand{\cR}{\mathcal{R}} %
\newcommand{\cB}{\mathcal{B}} %
\newcommand{\cJ}{\mathcal{J}} %
 
\newcommand{\cU}{\mathcal{U}} %
\newcommand{\cL}{\mathcal{L}} %
\newcommand{\UniformI}{\cU_{\cX_i}}
\newcommand{\Uniform}{\cU_{\cX}}
\newcommand{\Uniformfun}{\cU_{\cX}(\cdot)}

\newcommand{\tphi}{\tilde{\phi}}
\newcommand{\parameterSet}{\Lambda} %
\newcommand{\qThetaRVX}{q(\rvbx; \Theta)} %
\newcommand{\qThetaX}{q(\svbx; \Theta)} %
\newcommand{\qThetaXt}{q(\svbx^{(t)}; \Theta)} %
\newcommand{\qThetafun}{q(\cdot; \Theta)} %
\newcommand{\qThetafuntrue}{q(\cdot; \ThetaStar)} %
\newcommand{\DensityDifference}{f_{\rvbx}( \svbx;\ThetaStar - \Theta )}
\newcommand{\DensityDifferencefun}{f_{\rvbx}( \cdot;\ThetaStar - \Theta )}
\newcommand{\infdiv}{D\infdivx}
\newcommand{\lambdaMin}{\lambda_{\min}}
\newcommand{\phiMax}{\phi_{\max}}

\newcommand{\Expectation}{\mathbb{E}}

\newcommand{\vect}{\text{vec}}
\newcommand{\hH}{\hat{H}}
\newcommand{\Prob}{\mathbb{P}}

\newcommand{\epsTwo}{\epsilon}

\newcommand{\epsFour}{\epsilon}

\newcommand{\deltaTwo}{\delta}
\newcommand{\deltaThree}{\delta}
\newcommand{\deltaFour}{\delta}

\newcommand{\hthetan}{\hat{\theta}_n}
\newcommand{\Indicator}{\mathds{1}}
\newcommand{\mapsfrom}{\mathrel{\reflectbox{\ensuremath{\mapsto}}}}

\let\svthefootnote\thefootnote
\newcommand\freefootnote[1]{%
	\let\thefootnote\relax%
	\footnotetext{#1}%
	\let\thefootnote\svthefootnote%
}

\title{On Computationally Efficient Learning of \\Exponential Family Distributions\freefootnote{An earlier version of this work was presented at the Neural Information Processing Systems Conference in December 2021 titled ``A Computationally Efficient Method for Learning Exponential Family Distributions'' \citep{ShahSG2021}. The current version improves the rate in finite sample guarantees to the parametric rate, i.e., ($n^{-1/2}$ from $n^{-1/4}$) and connects our method with score-based methods as well as with non-parametric density estimation. This version also demonstrates the effectiveness of our estimator via numerical experiments.}}

\author{%
	Abhin Shah
	\\
	MIT\\
	\texttt{abhin@mit.edu} \\
	\and
	Devavrat Shah \\
	MIT \\
	\texttt{devavrat@mit.edu}
	\and
	Gregory W. Wornell
	\\
	MIT \\
	\texttt{gww@mit.edu}
}
\date{}
\begin{document}
\sloppy
\maketitle
\begin{abstract}
 We consider the classical problem of learning, with arbitrary accuracy, the natural parameters of a $k$-parameter truncated \textit{minimal} 
exponential family from i.i.d. samples in a computationally and statistically efficient manner. 
We focus on the setting where the support as well as the natural parameters are appropriately bounded.
While the traditional maximum likelihood estimator for this class of 
exponential family is consistent, asymptotically normal, and asymptotically efficient, evaluating it is computationally hard.
In this work, we propose a novel loss function and a computationally efficient estimator that is consistent 
as well as asymptotically normal under mild conditions. 
We show that, at the population level, our method can be viewed 
as the maximum likelihood estimation of a re-parameterized distribution belonging to the same 
class of exponential family. 
Further, we show that our estimator can be interpreted as a solution to minimizing a particular Bregman score as well as an instance of minimizing the \textit{surrogate} likelihood. 
We also provide finite sample guarantees to achieve an error (in $\ell_2$-norm) of $\alpha$ in the parameter estimation 
with sample complexity $O({\sf poly}(k)/\alpha^2)$. Our method achives the order-optimal sample complexity of  $O({\sf log}(k)/\alpha^2)$ when tailored for node-wise-sparse Markov random fields. Finally, we demonstrate the performance of our estimator via numerical experiments.%
\end{abstract}
\section{Introduction}
\label{sec_intro}%

Consider a $p$-dimensional random vector $\rvbx = (\rvx_1 , \cdots, \rvx_p)$ with support $\cX \subset \Reals^{p}$. 
An exponential family is a set of parametric probability distributions with probability densities of the following canonical form
\begin{align}
\DensityExpfam \propto \exp \big(\btheta^T \bphi(\svbx) + \beta(\svbx) \big), \label{eq:exp_fam}
\end{align} 
where $\svbx \in \cX$ is a realization of the underlying random variable $\rvbx$, $\btheta \in \Reals^{k}$ is the natural parameter, $\bphi : \cX \rightarrow \Reals^{k}$ is the natural statistic, $k$ denotes the number of parameters, and $\beta$ is the log base function. 

The notion of an exponential family was first intoduced by  \cite{Fisher1934} and was later generalized by  \cite{Darmois1935},  \cite{Koopman1936}, and  \cite{Pitman1936}. Exponential families play an important role in statistical inference and arise in many diverse applications for a variety of reasons: (a) they are analytically tractable, (b) they arise as the solutions to several natural optimization problems on the space of probability distributions, (c) they have robust generalization property (see \cite{Brown1986, BarndorffNielsen2014} for details).
\textit{Truncated} (or bounded) exponential family, first introduced by  \cite{HoggC1956}, is a set of parametric probability distributions resulting from truncating the support of an exponential family. \textit{Truncated} exponential families share the same parametric form with their non-truncated counterparts up to a normalizing constant. These distributions arise in many applications where we can observe only a truncated dataset (truncation is often imposed by during data acquisition) e.g., geolocation tracking data can only be observed up to the coverage of mobile signal, police department can often monitor crimes only within their city's boundary.

In this work, we are interested in learning the natural parameters of a \textit{minimal} exponential family with bounded support. 
Motivated by the kind of constraints on the natural parameters we focus on, an
equivalent representation of \eqref{eq:exp_fam} is:
\begin{align}
\DensityX \propto \exp\Big(  \big\langle\big\langle \Theta, \Phi(\svbx) \big\rangle \big\rangle\Big),
\end{align}
where $\Theta = [\Theta_{ij}] \in \Reals^{k_1\times k_2}$ is the natural parameter, $\Phi = [\Phi_{ij}] : \cX \rightarrow \Reals^{k_1\times k_2}$ is the natural statistic, $k_1 \times k_2 - 1= k$, and $\big\langle\big\langle\Theta, \Phi(\svbx)  \big\rangle\big\rangle$ denotes the matrix inner product, i.e., the sum of product of entries of $\Theta$ and $\Phi(\svbx)$. In other words, we utilize the following equivalent representation of \eqref{eq:exp_fam}:
\begin{align}
\DensityX \propto \exp\bigg(  \sum\limits_{i \in [k_1], j \in [k_2]} \Theta_{ij} \times \Phi_{ij}(\svbx) \bigg).  \label{eq:densityXfun}
\end{align}
An exponential family is \textit{minimal} if there does not exist a nonzero matrix $\bU \in \Reals^{\dimOne\times  \dimTwo} $ 
such that $\big\langle\big\langle \bU, \Phi(\svbx) \big\rangle\big\rangle$ is equal to a constant for all $\svbx \in \cX$. 

The natural parameter $\Theta$ specifies a particular distribution in the exponential family. If the natural statistic $\Phi$ and the support of $\rvbx$ (i.e., $\cX$) are known, then learning a distribution in the exponential family is equivalent to learning the corresponding natural parameter $\Theta$. Despite having a long history, there has been limited progress on learning natural parameter $\Theta$ of a \textit{minimal truncated} exponential family. 
Motivated by this, we address the following fundamental question in this work --- \textit{Is there a method that is both computationally and statistically efficient for learning natural parameter of the \textit{minimal truncated} exponential family considered above?}

\subsection{Contributions}
\label{subsec_contributions}
As the primary contribution of this work, we provide a computationally tractable method with statistical guarantees for learning distributions in \textit{truncated minimal} exponential families.
Formally, the learning task of interest is estimating the true natural parameter $\ThetaStar$ from i.i.d. samples of $\rvbx$ obtained from $\DensityXTruefun$.
We focus on the setting where the natural parameter $\ThetaStar$ belongs to a convex set and the natural statistics $\Phi$ is appropriately bounded (see Section \ref{sec:prob_formulation}).
 We summarize our contributions in the following two categories.\\

\subsubsection{Estimator: Computational Tractability, Consistency, Asymptotic Normality, Finite Sample Guarantees}
Given $n$ samples $\svbx^{(1)} \cdots , \svbx^{(n)}$ of $\rvbx$, we propose the following convex loss function to learn a distribution belonging to the exponential family in \eqref{eq:densityXfun}:
\begin{align}
\cL_{n}(\Theta)  = \frac{1}{n} \sum_{t = 1}^{n} \exp\Big( -\big\langle \big\langle \Theta, \varPhi(\svbx^{(t)}) \big\rangle \big\rangle \Big), \label{eq:loss_function}
\end{align}
where $\varPhi(\cdot) = \Phi(\cdot) - \Expectation_{\Uniform} [\Phi(\cdot)]$ with $\Uniform$ being the uniform distribution over $\cX$. 
We establish that the estimator $\hThetan$ obtained by minimizing $\cL_{n}(\Theta)$ over all $\Theta$ in the constraint set $\parameterSet$, i.e., 
\begin{align}
\hThetan \in \argmin_{\Theta \in \parameterSet} \cL_{n}(\Theta), \label{eq:estimator}
\end{align}
is consistent and (under mild further restrictions) asymptotically normal (see Theorem \ref{thm:consistency_normality}). We show that the loss function in \eqref{eq:loss_function} is $O(\dimOne \dimTwo)$-smooth function of $\Theta$ implying that an $\epsilon$-optimal solution $\hThetaEps$ (i.e., $\cL_{n}(\hThetaEps) \leq \cL_{n}(\hThetan) + \epsilon$) can be obtained in $O(k_1k_2/\epsilon)$ iterations. Finally, we provide rigorous finite sample guarantees for $\hThetan$ to 
achieve an error of $\alpha$ (in the Frobenius norm) with respect to the true
natural parameter $\ThetaStar$ with $O({\sf poly}(k_1k_2)/\alpha^2)$ samples (see Theorem \ref{thm:finite_sample}). We note the loss function in \eqref{eq:loss_function} is a generalization of the loss functions proposed in \cite{VuffrayMLC2016,VuffrayML2019,ShahSW2021} for learning node-wise-sparse Markov random fields (MRFs). In particular, Theorem \ref{thm:finite_sample} could be specialized to recover the natural parameters of sparse MRFs with $O({\sf log}(k_1k_2)/\alpha^2)$ samples (see Remark \ref{remark_mrf}). Beyond local structure on $\Theta$ (of node-wise-sparsity) as in MRFs, our framework can capture various global structures on $\Theta$, e.g., bounded maximum norm, bounded Frobenius norm, bounded nuclear norm (see Section \ref{subsec:natural_parameter}).\\

\subsubsection{Connections to Maximum Likelihood Estimation, Bregman Score, and Surrogate Likelihood}
We establish connections between our method and various existing methods in the literature. First,
we show that the estimator that minimizes the population version of the loss function in \eqref{eq:loss_function} i.e.,
$$\cL(\Theta)  = \Expectation \Big[\exp\Big( -\big\langle \big\langle \Theta, \varPhi(\rvbx) \big\rangle \big\rangle \Big)\Big],$$
is equivalent to the estimator that minimizes the  Kullback-Leibler (KL) divergence between $\Uniform$ (the uniform distribution on $\cX$) and $\DensityDifferencefun$ (see Theorem \ref{theorem:GRISMe-KLD}). Therefore, at the population level, our method can be viewed as the MLE of the parametric family $\DensityDifferencefun$.
We show that the KL divergence (and therefore $\cL(\Theta)$) is minimized if and only if 
$\Theta = \ThetaStar$ implying that $\cL(\Theta)$ is a proper loss function. This connection provides an intuitively pleasing justification of the estimator in \eqref{eq:estimator}.
Second, we demonstrate that the estimator in \eqref{eq:estimator} can be interpreted as a solution to minimizing a particular Bregman score and thus connect our method with score-based methods (see Proposition \ref{prop:bregman_score}). In other words, we show that a specific \textit{separable} Bregman score is equivalent to the learning task considered here. Therefore, our work also implies the computational tractability of this Bregman score.
Finally, we show that our estimator can be viewed as an instance of the \textit{surrogate} likelihood estimator proposed by \cite{JeonL2006}, and thus draw connections with non-parametric density estimation (see Proposition \ref{prop:surrogate}).

 \subsection{Comparison with the traditional MLE}
 To contextualize our method, we compare it with the MLE of the parametric family $\DensityXfun$. The MLE of $\DensityXfun$ minimizes the following loss function 
 \begin{align}
\min - \frac{1}{n} \sum_{t = 1}^{n} \big\langle \big\langle  \Theta, \Phi(\svbx^{(t)}) \big\rangle \big\rangle + \log \int_{\svbx \in \cX} \hspace{-4mm} \exp\Big(\big\langle \big\langle \Theta, \Phi(\svbx \big\rangle\big\rangle\Big) d\svbx. \label{eq:mle}
 \end{align}
The maximum likelihood estimator has many attractive asymptotic properties 
 : (a) consistency \citep[Theorem~17]{Ferguson2017}, i.e., as the sample size goes to infinity, the bias in the estimated parameters goes to zero, (b) asymptotic normality  \citep[Theorem~18]{Ferguson2017}, i.e., as the sample size goes to infinity, normalized estimation error coverges to a Gaussian distribution
and (c) asymptotic efficiency \citep[Theorem~20]{Ferguson2017}, i.e., as the sample size goes to infinity, the variance 
in the estimation error
attains the minimum possible value among all consistent estimators.
Despite having these useful asymptotic properties of consistency, normality, and efficiency, 
 computing the partition function in \eqref{eq:mle}, and therefore, the maximum likelihood estimator, is computationally hard \citep{Valiant1979, JerrumS1989}. Furthermore, even approximating the partition function, up to a multiplicative error, is NP-hard in general \citep{sly2012computational}.
 
Our method can be viewed as a computationally efficient proxy for the MLE. More precisely, our method is computationally tractable as opposed to the MLE while retaining the useful properties of consistency and asymptotic normality. However, our method misses out on asymptotic efficiency. This raises an important question for future work --- \textit{can computational and asymptotic efficiency be achieved by a single estimator for this class of exponential family?}

\subsection{Related Works}
\label{subsec_related_work}
In this section, we look at the related works on learning exponential family. First, we look at the two broad line of approaches to overcome the computational hardness of the MLE : (a) approximating the MLE and (b) selecting a surrogate objective. Given the richness of both of approaches, we cannot do justice in providing a full overview. Instead, we look at a few examples from both. Next, we look at some of the related works that focus on learning a class of exponential family. More specifically, we look at works on (a) learning the Gaussian distribution and (b) learning exponential family Markov random fields (MRFs). 
Next, we explore some works on the powerful technique of score matching. Finally, we review the related literature on non-parametric density estimation to draw the connection between our estimator and the {surrogate} likelihood estimator \citep{JeonL2006}.
In Appendix \ref{appendix:related_works}, we further review works on learning exponential family MRFs,
score-based methods (including the related literature on Stein discrepancy), and non-parametric density estimation.\\
\noindent{\bf Approximating the MLE.} 
Most of the techniques falling in this category approximate the MLE by approximating the log-partition function. A few examples include : (a) approximating the gradient of log-likelihood with a stochastic estimator by minimizing the contrastive divergence \citep{Hinton2002}; (b) upper bounding the log-partition function by an iterative tree-reweighted belief propagation algorithm \citep{WainwrightJW2003}; (c) using Monte Carlo methods like importance sampling for estimating the partition function \citep{RobertC2013}. Since these methods approximate the partition function, they come at the cost of an approximation error or result in a biased estimator.\\

\noindent{\bf Selecting surrogate objective.} 
This line of approach selects an easier-to-compute surrogate objective that completely avoids the partition function. A few examples are as follows : (a) pseudo-likelihood estimators \citep{Besag1975} approximate the joint distribution with the product of conditional distributions, each of which only represents the distribution of a single variable conditioned on the remaining variables; 
(b) score matching \citep{HyvarinenD2005, Hyvarinen2007} minimizes the Fisher divergence between the true log density and the model log density. Even though score matching does not require evaluating the partition function, it is computationally expensive as it requires computing third order derivatives for optimization; (c) kernel Stein discrepancy \citep{LiuLJ2016, ChwialkowskiSG2016} measures the kernel mean discrepancy between a data distribution and a model density using the Stein's identity. This measure is directly characterized by the choice of the kernel and there is no clear objective for choosing the right kernel \citep{WenliangSSG2019}.\\

\noindent{\bf Learning the Gaussian distribution.}
Learning the Gaussian distribution is a special case of learning exponential family distributions. There has been a long history of learning Gaussian distributions in the form of learning Gaussian graphical models, e.g., the neighborhood selection scheme \citep{MeinshausenB2006}, the graphical lasso \citep{FriedmanHT2008}, the CLIME \citep{CaiLL2011}, etc. However, finite sample analysis of these methods require various hard-to-verify conditions, e.g., the restricted eigenvalue condition, the incoherence assumption \citep{WainwrightRL2006, JalaliRVS2011}, bounded eigenvalues of the precision matrix. \cite{misra2020information} provided the first polynomial-time algorithm whose sample complexity matches the information-theoretic lower bound of \cite{wang2010information} without the aforementioned hard-to-verify conditions. \cite{KelnerKMM2019} proposed a faster alternative to \cite{misra2020information} for a specific subclass of Gaussian graphical models.
Similarly, there has also been a long history of learning truncated Gaussian distributions dating back to \cite{Galton1898, Pearson1902, PearsonL1908, Lee1914, Fisher1931}. 
A recent work by \cite{DaskalakisGTZ2018} showed that it is possible to learn, in polynomial time, the mean vector and the covariance matrix of a $p$-dimensional Gaussian distribution, upto an ($\ell_2$) error of $\alpha$ with $O(p^2/\alpha^2)$ samples (i.e., with a sample complexity of the same order when there is no truncation).\\

\noindent{\bf Learning Exponential family Markov random fields (MRFs).}
MRFs can be naturally represented as exponential family distributions via the principle of maximum entropy \citep{WainwrightJ2008}. A popular method for learning  MRFs is estimating node-neighborhoods (fitting conditional distributions of each node conditioned on the rest of the nodes) because the natural parameter is assumed to be node-wise-sparse. A recent line of work has considered a subclass of node-wise-sparse pairwise continuous MRFs where the node-conditional distribution of $\rvx_i \in \cX_i$ for every $i$ arise from an exponential family as follows:
\begin{align}
f_{\rvx_i | \rvx_{-i}}(x_i | \svbx_{-i} = x_{-i}) \propto \exp \Big( \big[  \theta_i  + \hspace{-2mm}  \sum_{j \in [p], j \neq i}  \hspace{-2mm} \theta_{ij}  \phi(x_j)  \big] \phi(x_i) \Big),  \label{eq:conditional}
\end{align}
where $\phi(x_i)$ is the natural statistics and $\theta_i  + \sum_{j \in [p], j \neq i} \theta_{ij}  \phi(x_j) $ is the natural parameter.\footnote[2]{Under node-wise-sparsity, $\sum_{j \in [p], j \neq i} \Indicator(\theta_{ij} \neq 0) $ is bounded by a constant for every $i \in [p]$.} \cite{YangRAL2015} showed that only the following joint distribution is consistent with the node-conditional distributions in \eqref{eq:conditional} :
\begin{align}
f_{\rvbx}(\svbx) \propto \exp \Big( \sum_{i \in [p]} \theta_i  \phi(x_i) + \sum_{j \neq i} \theta_{ij}  \phi(x_i) \phi(x_j) \Big).  \label{eq:joint}
\end{align}
To learn the node-conditional distribution in \eqref{eq:conditional} for linear $\phi(\cdot)$ (i.e., $\phi(x) = x$), \cite{YangRAL2015} proposed an $\ell_1$ regularized node-conditional log-likelihood. However, their
finite sample analysis 
required the following conditions: incoherence, dependency  \citep{WainwrightRL2006, JalaliRVS2011}, bounded moments of the variables, and local smoothness of the log-partition function. \cite{TanseyPSR2015} extended the approach in \cite{YangRAL2015} to vector-space MRFs (i.e., vector natural parameters and natural statistics) and non-linear $\phi(\cdot)$. They proposed a sparse group lasso \citep{ SimonFHT2013} regularized node-conditional log-likelihood and an alternating direction method of multipliers based approach to solving the resulting optimization problem. However, their analysis required same conditions as \cite{YangRAL2015}. 

While node-conditional log-likelihood has been a natural choice for learning exponential family MRFs in \eqref{eq:joint}, M-estimation \citep{VuffrayMLC2016, VuffrayML2019, ShahSW2021} and maximum pseudo-likelihood estimator \citep{NingZL2017, YangNL2018, DaganDDA2021} have recently gained popularity. The objective function in M-estimation is a sample average and the estimator is generally consistent and asymptotically normal. 
\cite{ShahSW2021} proposed the following M-estimation (inspired from \cite{VuffrayMLC2016, VuffrayML2019}) for vector-space MRFs and non-linear $\phi(\cdot)$: with $\UniformI$ being the uniform distribution on $\cX_i$ and $\tphi(x_i) = \phi(x_i) - \int_{x_i'} \phi(x_i') \UniformI(x_i') dx_i'$
\begin{align}
\arg \min \frac{1}{n} \sum_{t = 1}^{n} \exp\Big(\!-\! \Big[\theta_i  \tphi(x_i^{(t)}) + \!\!\!\!\! \sum_{j \in [p], j \neq i} \hspace{-2mm} \theta_{ij}  \tphi(x_i^{(t)}) \tphi(x_j^{(t)}) \Big] \Big). \label{eq:shah}
\end{align}
They provided an entropic descent algorithm (borrowing from \cite{VuffrayML2019}) to solve the optimization in \eqref{eq:shah} and their finite-sample bounds rely on bounded domain of the variables and a condition (naturally satisfied by linear $\phi(\cdot)$) that lower bounds the variance of a non-constant random variable. 

\cite{YuanLZLL2016} considered a broader class of sparse pairwise exponential family MRFs compared to \cite{YangRAL2015}, i.e., \eqref{eq:joint}. They studied the following joint distribution with natural statistics $\phi(\cdot)$ and $\psi(\cdot)$
\begin{align}
f_{\rvbx}(\svbx) \propto \exp \Big( \sum_{i \in [p]} \theta_i  \phi(x_i) + \sum_{j \neq i} \theta_{ij}  \psi(x_i, x_j) \Big).  \label{eq:yuan}
\end{align}
They proposed an $\ell_{2,1}$ regularized joint likelihood and an $\ell_{2,1}$ regularized node-conditional likelihood. They also presented a Monte-Carlo approximation to these estimators via proximal gradient descent. Their finite-sample analysis required restricted strong convexity (of the Hessian of the negative log-likelihood of the joint density) and bounded moment-generating function of the variables. 

Building upon \cite{VuffrayMLC2016, VuffrayML2019} and \cite{ShahSW2021}, \cite{RenMVL2021} addressed learning continuous exponential family distributions through a series of numerical experiments. They considered unbounded distributions and allowed for terms corresponding to multi-wise interactions in the joint density. 
They 
assume local structure on the parameters as in MRFs and their
estimator is defined as a series of node-wise optimization problems. We note that, among the methods described above, the ones based on M-estimation have superior numerical performance compared to the ones based on pseudo-likelihood estimation as shown in \cite{RenMVL2021}. Likewise, building upon \cite{ShahSW2021} and \cite{DaganDDA2021}, \cite{shah2022counterfactual} considered learning the node-conditional distribution corresponding to the joint distribution in \eqref{eq:joint} when certain variables remain unobserved. They showed applications to causal inference and provided finite sample guarantees for learning the counterfactual means. 

In summary, tremendous progress has been made on learning the sub-classes of exponential family in \eqref{eq:joint} and \eqref{eq:yuan}. However, this sub-classes are restricted by the assumption that the natural parameters are node-wise-sparse. For example, none of the existing methods for exponential family MRFs work in the setting where the natural parameters have a nuclear norm (i.e., a convex relaxation of a low-rank) constraint.\\
\noindent{\bf Score-based method.}
A scoring rule $S(\svbx, Q)$ is a numerical score assigned to a realization $\svbx$ of a random variable $\rvbx$ and it measures the quality of a predictive distribution $Q$ (with probability density $q(\cdot)$). 
If $P$ is the true distribution of $\rvbx$, the divergence $D(P,Q)$ associated with a scoring rule is defined as $\Expectation_{P}[S(\rvbx, Q) - S(\rvbx, P)]$. The MLE is an example of a scoring rule with $S(\cdot, Q) = - \log q(\cdot)$ and the resulting divergence is the KL-divergence.

To bypass the intractability of MLE, \cite{HyvarinenD2005, Hyvarinen2007} proposed an alternative scoring rule with $S(\cdot, Q) = \Delta \log q(\cdot) + \frac{1}{2} \| \nabla \log q(\cdot)\|^2_2$ where $\Delta$ is the Laplacian operator, $\nabla$ is the gradient and $\| \cdot \|_2$ is the $\ell_2$ norm. This method is called \textit{score matching} and the resulting divergence is the Fisher divergence. Score matching is widely used for estimating unnormalizable probability distributions because computing the scoring rule $S(\cdot, Q)$ does not require knowing the partition function. Despite the flexibility of this approach, it is computationally expensive in high dimensions since it requires computing the trace of the unnormalized density's Hessian (and its derivatives for optimization). Additionally, it breaks down for models in which the second derivative grows very rapidly.

\cite{LiuKW2019} considered estimating truncated exponential family using the principle of {score matching}. They build on the framework of generalized score matching \citep{Hyvarinen2007} and proposed a novel estimator that minimizes a weighted Fisher divergence. They showed that their estimator is a special case of minimizing a Stein Discrepancy. However, their finite sample analysis relies on certain hard-to-verify assumptions, for example, the assumption that the optimal parameter is well-separated from other neighboring parameters in terms of their population objective. Further, their estimator lacks the useful properties of asymptotic normality and asymptotic efficiency.\\

\noindent{\bf Non-parametric density estimation.}
	In their pioneering article, \cite{GooddG1971} introduced the idea of	penalized log likelihood for non-parametric density estimation. 
	The approach of logistic density transform, commonly used today for non-parametric density estimation, was first introduced by \cite{Leonard1978} to incorporate the positivity $(f_{\rvbx}(\cdot) \geq 0)$ and unity $(\int_{\svbx \in \cX} f_{\rvbx} (\svbx) d\svbx  = 1)$ constraints of density function. They considered densities of the form
	\begin{align}
	 f_{\rvbx}(\svbx) = \frac{\exp(\eta(\svbx))}{\int_{\svbx \in \cX} \exp(\eta(\svbx)) d\svbx},
	\end{align}
	with some constraints on $\eta(\cdot)$ for identifiability, and proposed to estimate $\eta(\cdot)$ via penalized log likelihood as follows
	\begin{align}
	\hspace{-6mm}  \arg \min - \frac{1}{n} \sum_{i = 1}^{n} \eta(\svbx_i) + \log \int_{\svbx \in \cX}  \hspace{-4mm}  \exp(\eta(\svbx)) d\svbx + \lambda J(\eta), \label{eq:leonard}
	\end{align}
	where $\lambda \geq 0$ is a smoothing parameter and $J(\eta)$ is a penalty functional.
	
	While the penalized likelihood method in \eqref{eq:leonard} has been successful in low-dimensions, it scales poorly in high-dimensions. In high-dimensional problems, the main difficulty is in computing multidimensional integrals of the form $\int_{\svbx \in \cX} \exp(\eta(\svbx)) d\svbx$ which do not decompose in general. To circumvent this computational limitation associated with \eqref{eq:leonard}, \cite{JeonL2006} proposed a penalized $M$-estimation (\textit{surrogate} likelihood) type method as follows
	\begin{align}
	\hspace{-6mm}  \arg \min \frac{1}{n} \sum_{i = 1}^{n} \exp(-\eta(\svbx_i)) +  \int_{\svbx \in \cX} \hspace{-4mm} \rho(\svbx) \eta(\svbx) d\svbx + \lambda J(\eta), \label{eq:jeon}
	\end{align}
	where $\eta(\cdot)$ lies in a Reproducing Kernel Hilbert Space (RKHS) and $\rho(\cdot)$ is some fixed known density with the same support as the unknown density $f_{\rvbx}(\svbx)$, with the resulting density estimate $\hat{f}_{\rvbx}(\svbx) = \rho(\svbx) \exp(\hat{\eta}(\svbx))$. They showed that with appropriate choices of $\rho(\cdot)$, the integral $\int_{\svbx \in \cX} \rho(\svbx) \eta(\svbx) d\svbx$ can be decomposed into sums of products of one-dimensional integrals, allowing faster computations. However, the selection of $\lambda$ that delivers reasonable performance requires the evaluation of the normalizing constant $\int_{\svbx \in \cX} \rho(\svbx) \exp(\eta(\svbx)) d\svbx $.  Additionally, the theoretical properties of the estimator in \eqref{eq:jeon} are also not yet known.

\subsection{Useful notations}
\label{subsec_notations}
For any positive integer $t$, let $[t] \coloneqq \{1,\cdots, t\}$.
For a deterministic sequence $v_1, \cdots , v_t$, we let $\svbv \coloneqq (v_1, \cdots, v_t)$. 
For a random sequence $\rvv_1, \cdots , \rvv_t$, we let $\rvbv \coloneqq (\rvv_1, \cdots, \rvv_t)$. We denote the $\ell_p$ norm $(p \geq 1)$ of a vector $\svbv \in \Reals^t$ by $\| \svbv\|_{p} \coloneqq (\sum_{i=1}^{t}|v_i|^p)^{1/p}$ and the $\ell_{\infty}$ norm by $\|\svbv\|_{\infty} \coloneqq \max_{i \in [t]} |v_i|$.  For a matrix $\bM \in \Reals^{u \times v}$, we denote the element in $i^{th}$ row and $j^{th}$ column by $M_{ij}$, and the singular values of the matrix by $\sigma_i(\bM)$ for $i \in [\min\{u, v\}]$. We denote the matrix maximum norm by $\|\bM\|_{\max} \coloneqq \max_{i \in [u], j \in [v]} |M_{ij}|$, the entry-wise $L_{1,1}$ norm by $\|\bM\|_{1,1} \coloneqq \sum_{i \in [u], j \in [v]} |M_{ij}|$, the nuclear norm by $\|\bM\|_{\star} \coloneqq \sum_{i \in [\min\{u, v\}]} \sigma_i(\bM)$,  the spectral norm by $\|\bM\| \coloneqq \max_{i \in [\min\{u, v\}]} \sigma_i(\bM)$, and the Frobenius norm by $\| \bM\|_{\mathrm{F}} \coloneqq \sqrt{\sum_{i \in [u], j \in [v]} M^2_{ij}}$. 
We denote the Frobenius or Trace inner product of matrices $\bM, \bN \in \Reals^{u \times v}$ by  $\langle\langle \bM, \bN \rangle\rangle \coloneqq \sum_{i \in [u], j \in [v]} M_{ij} N_{ij}$. 
For a matrix $\bM \in \Reals^{u \times v}$, we denote a generic norm on $\Reals^{u \times v}$ by $\cR(\bM)$ and denote the associated dual norm by $\cR^*(\bM) \coloneqq \sup \{\langle\langle \bM, \bN \rangle\rangle | \cR(\bN) \leq 1\}$ where $\bN \in \Reals^{u \times v}$. 
We denote the vectorization of a matrix $\bM \in \Reals^{u \times v}$ by $\vect(\bM) \in \Reals^{uv \times 1}$ (the ordering of the elements is not important as long as it is consistent).
Let $\boldsymbol{0} \in \Reals^{\dimOne \times \dimTwo}$ denote the matrix with every entry zero. We denote a $p$-dimensional $\ell_q$ ball of radius $b$ centered at $\boldsymbol{0} \in \Reals^{p}$ by $\cB_q(b)$ for $q \in \{1,2\}$.

\subsection{Outline}
In Section \ref{sec:prob_formulation}, we formulate the problem of interest and provide examples. In Section \ref{sec:algorithm}, we provide our loss function. In Section \ref{sec:main results}, we present our main results including the connections to the MLE of $\DensityDifferencefun$, 
consistency, asymptotic normality, and finite sample guarantees. 
In Section \ref{sec:experiments}, we provide our empirical findings. In Section \ref{sec:misc}, we conclude and provide some remarks.

\section{Problem Formulation}
\label{sec:prob_formulation}
Let $\rvbx = (\rvx_1 , \cdots, \rvx_p)$ be a $p-$dimensional vector of continuous random variables.\footnote[3]{Even though we focus on continuous variables, our framework applies equally to discrete or mixed variables.} For any $i \in [p]$, let the support of $\rvx_i$ be $\cX_i \subset \Reals$.
Define $\cX \coloneqq \prod_{i=1}^p \cX_i$.
Let $\svbx = (x_1, \cdots, x_p) \in \cX$ be a realization of $\rvbx$. In this work, we assume that the random vector $\rvbx$ belongs to an exponential family with bounded support (i.e., length of $\cX_i$ is bounded) along with certain additional constraints. 

\subsection{Natural parameter $\Theta$}
\label{subsec:natural_parameter}
First, we assume that a certain norm of the natural parameter $\Theta \in \Reals^{\dimOne \times  \dimTwo}$ is bounded. This serves to confine the possible values of $\Theta$ within a convex set. We formally state this assumption below. 
\begin{assumption}\label{bounds_parameter}
	(Bounded norm of $\Theta$.)
	We let $\cR(\Theta) \leq r$ where $\cR : \Reals^{\dimOne\times  \dimTwo} \rightarrow \Reals_+$ is a norm and $r$ is a known constant.
\end{assumption}
Assumption \ref{bounds_parameter} should be viewed as a potential flexibility in the problem specification i.e., a practitioner has the option to choose from a variety of constraints on the natural parameters (that could be handled by our framework). For example, in some real-world applications every entry of the parameter is bounded while in some other case the sum of singular values of the parameter matrix is bounded (convex relaxation of low-rank parameter matrix). A few concrete examples include bounded maximum norm $\|\Theta\|_{\max}$, 
bounded Frobenius norm $\|\Theta\|_{\mathrm{F}}$, and bounded nuclear norm $\|\Theta\|_{\star}$. 

\begin{remark}
	We note that $p$, the dimension of $\rvbx$, is not assumed to be a constant. Instead, we think of $\dimOne$ and $\dimTwo$ as implicit functions of $p$. Typically, for an exponential family, the quantity of interest is the number of parameters, i.e., $k = \dimOne \times \dimTwo$, and this quantity scales polynomially in $p$, e.g., $k = O(p^t)$ for t-wise MRFs over binary alphabets (see Section \ref{sec:experiments}). 
\end{remark}

Now, we define $\parameterSet$ to be the set of all natural parameters satisfying Assumption \ref{bounds_parameter} i.e., $\parameterSet \coloneqq \{\Theta : \cR(\Theta) \leq \br\}$. To see that the constraint set $\parameterSet$ is a convex set, consider any $\tilde{\Theta}, \bar{\Theta} \in \parameterSet$ and $t \in [0,1]$. Then, we have $\cR(t \tilde{\Theta} + (1-t) \bar{\Theta} ) \leq t \cR(\tilde{\Theta}) + (1-t) \cR(\bar{\Theta} ) \leq t\br + (1-t)\br = \br$ which implies $t \tilde{\Theta} + (1-t) \bar{\Theta}  \in \parameterSet$. Next, we make certain assumptions on the natural statistic $\Phi(\svbx) : \cX \rightarrow \Reals^{\dimOne \times  \dimTwo}$.
\subsection{Natural Statistic $\Phi$}
We focus on bounded natural statistics and consider two notions of boundedness.  First, we assume that the dual norm (defined with respect to $\cR$) of the natural statistic is bounded. This enables us to bound the matrix inner product between the natural parameter $\Theta$ and the natural statistic $\Phi(\cdot)$.
\begin{assumption}\label{bounds_statistics}
	(Bounded dual norm of $\Phi$). Let $\cR^*$ denote the dual norm of $\cR$. Then, we  assume that the dual norm $\cR^*$ of the  natural statistic $\Phi$ is bounded by a constant $d$. Formally, for any $\svbx \in \cX$, $\cR^*(\Phi(\svbx)) \leq d$. 
\end{assumption}
A few examples of dual norms include the $L_{1,1}$ norm $\|\Phi(\svbx)\|_{1,1}$, the Frobenius norm $\|\Phi(\svbx)\|_{\mathrm{F}}$, and the spectral norm $\|\Phi(\svbx)\|$ when the underlying norm $\cR$ is the maximum norm $\|\Theta\|_{\max}$, the Frobenius norm $\|\Theta\|_{\mathrm{F}}$, and the nuclear norm $\|\Theta\|_{\star}$, respectively. While we require this assumption for our theoretical analysis, our empirical findings suggest that the assumption may not be a strict requirement in practice (see Section \ref{sec:experiments}).  
Next, we assume that the maximum norm of the natural statistic $\Phi(\cdot)$ is also bounded. This assumption is stated formally below.
\begin{assumption}\label{bounds_statistics_maximum}
	(Bounded maximum norm of $\Phi$). For any $\svbx \in \cX$, $\|\Phi(\svbx)\|_{\max} \leq \phiMax$ where $\phiMax$ is a constant.
\end{assumption}

\subsection{The Exponential Family}
Summarizing, $\rvbx$ belongs to a \textit{minimal truncated} exponential family with probability density function as follows
\begin{align}
\DensityX \propto \exp\Big(  \Big\langle \Big\langle \Theta, \Phi(\svbx) \Big\rangle \Big\rangle \Big), \label{eq:densityX}
\end{align}
where the natural parameter $\Theta \in \Reals^{\dimOne\times  \dimTwo}$ is such that $\cR(\Theta) \leq r$ for some norm $\cR$, and the natural statistic $\Phi(\svbx) : \cX \rightarrow \Reals^{\dimOne\times  \dimTwo}$ is such that for any $\svbx \in \cX$,  $\|\Phi(\svbx)\|_{\max} \leq \phiMax$ as well as $\cR^*(\Phi(\svbx)) \leq \bd$ for the dual norm $\cR^*$.

Let $\ThetaStar$ denote the true natural parameter of interest and $\DensityXTrue$ denote the true distribution of $\rvbx$. Naturally, we assume $\cR(\ThetaStar) \leq r$. Formally, the learning task of interest is as follows:
Given $n$ independent samples of $\rvbx$ i.e., $\svbx^{(1)} \cdots , \svbx^{(n)}$ obtained from $\DensityXTrue$, compute an estimate $\hTheta$ of $\ThetaStar$ in polynomial time such that $\|\ThetaStar - \hTheta\|_{\mathrm{F}}$ is small.
\subsection{Examples}
\label{subsec:examples}
Next, we present a few examples of natural statistics along with the corresponding support that satisfy Assumptions \ref{bounds_statistics} and \ref{bounds_statistics_maximum}. See Appendix \ref{appendix:examples} for more discussion on these examples.

\begin{enumerate}[leftmargin=*,topsep=1pt,itemsep=1pt]
	\item \textit{Polynomial statistics}: Suppose the natural statistics are polynomials of $\rvbx$ with maximum degree $l$, i.e., $\prod_{i \in [p]} x_i^{l_i}$ such that $l_i \in [l] \cup \{0\}$ for all $ i \in [p]$ and $\sum_{i \in [p]} l_i \leq l$ for some $l < p$. If $\cX = [-b,b]^p$ for some $b \in \Reals$, then $\phiMax = \max\{1, b^l\}$. If $\ThetaStar$ has a bounded maximum norm $\|\Theta\|_{\max}$ and $\cX = \cB_1(b)$ for $b \in \Reals$, then $\cR^*(\Phi(\svbx)) \leq (1 + b)^l$. If $\ThetaStar$ has a bounded Frobenius norm $\|\Theta\|_{\mathrm{F}}$ and $\cX = \cB_1(b)$ for $b \in \Reals$, then $\cR^*(\Phi(\svbx)) \leq (1 + b)^l$. If $\ThetaStar$ has a bounded nuclear norm $\|\Theta\|_{\star}$ , $l = 2$, and $\cX = \cB_2(b)$ for $b \in \Reals$, then $\cR^*(\Phi(\svbx)) \leq b^2$. 	
	\item \textit{Trigonometric statistics}: Suppose the natural statistics are sines and cosines of $\rvbx$ with $l$ different frequencies, i.e., $\sin(\sum_{i \in [p]}l_ix_i)$ $\cup$ $\cos(\sum_{i \in [p]}l_ix_i)$ such that $l_i \in [l] \cup \{0\}$ for all $i \in [p]$. For any $\cX \subset \Reals^p$, $\phiMax = 1$. If $\ThetaStar$ has bounded $L_{1,1}$ norm $\|\Theta\|_{1,1}$, then $\cR^*(\Phi(\svbx)) \leq 1$ for any $\cX \subset \Reals^p$.
\end{enumerate}
Our framework also allows combinations of polynomial and trigonometric statistics (see Appendix \ref{appendix:examples}).
\section{Loss function}
\label{sec:algorithm} 
We propose a novel and computationally tractable loss function
drawing inspiration from the recent advancements in exponential family Markov random fields \citep{VuffrayMLC2016, VuffrayML2019, ShahSW2021}. We start by centering the natural statistic $\Phi(\cdot)$ such that their integral with respect to the uniform distribution on $\cX$, denoted by $\Uniform$, is zero. We note that the distribution $\Uniform$ is well-defined as the support $\cX$ is a strict subset of $\Reals^{p}$ i.e., $\cX \subset \Reals^{p}$. This centering plays a key role in ensuring that our loss function is a proper loss function.
\begin{definition}\label{def:css}
	(Centered natural statistics). 
	The centered natural statistics are defined as follows: 
	\begin{align}
	\varPhi(\cdot) & \coloneqq \Phi(\cdot) - \Expectation_{\Uniform} [\Phi(\rvbx)]. \label{eq:CSS}
	\end{align}
\end{definition}

The loss function, defined below, is 
an empirical average of the inverse of the function of $\rvbx$ that the probability density $\DensityX$ in \eqref{eq:densityX} is proportional to.
\begin{definition}[The loss function] Given $n$ samples $\svbx^{(1)} \cdots , \svbx^{(n)}$ of $\rvbx$, the loss function maps $\Theta \in \Reals^{\dimOne\times  \dimTwo}$ to $\cL_{n}(\Theta) \in \Reals$ defined as 
	\begin{align}
	\cL_{n}(\Theta)  = \frac{1}{n} \sum_{t = 1}^{n} \exp\big( -\big\langle \big\langle \Theta, \varPhi(\svbx^{(t)}) \big\rangle\big\rangle \big). \label{eq:sampleGISMe}
	\end{align}
\end{definition}

The proposed estimator $\hThetan$ produces an estimate of $\ThetaStar$ by minimizing the loss function $\cL_{n}(\Theta)$ over all natural parameters $\Theta$ satisfying Assumption \ref{bounds_parameter} i.e.,
\begin{align}
\hThetan \in \argmin_{\Theta \in \parameterSet} \cL_{n}(\Theta).
 \label{eq:GRISMe}
\end{align}
The optimization in \eqref{eq:GRISMe} is a convex minimization problem, i.e., minimizing a convex function $\cL_{n}$ over a convex set $\parameterSet$). As a result, there are efficient (i.e., polynomial time) implementations for finding an $\epsilon$-optimal solution of $\hThetan$ where $\hThetaEps$ is said to be an $\epsilon$-optimal solution of $\hThetan$ if $\cL_{n}(\hThetaEps) \leq \cL_{n}(\hThetan) + \epsilon$ for any $\epsilon > 0$. The run-time of these implementations scales as $O(\dimOne \dimTwo/\epsilon)$ which follows from the $O(\dimOne \dimTwo)$-smoothness property of $\cL_{n}(\Theta)$ (see Appendix \ref{appendix:algorithm_section_proofs} for a proof). Furthermore, from Slater's condition (which holds because $\parameterSet$ has an interior point), we can express $\hThetan$ as a solution to the following unconstrained optimization: 
\begin{align}
\hThetan \in \argmin_{\Theta \in \Reals^{\dimOne \times \dimTwo}} \cL_{n}(\Theta) + \lambda_n \cR(\Theta),
\label{eq:GRISMe_unconstrained}
\end{align}
where $\lambda_n$ is regularization penalty. As we show in our main results,  $\hThetan$ is close (in Frobenius norm) to $\ThetaStar$ when $\lambda_n$ is appropriately chosen. We note that addition of the regularization preserves convexity of the optimization in \eqref{eq:GRISMe_unconstrained} due to convexity of norms. In fact, depending on the exact form of $\cR(\Theta)$, the function $\cL_{n}(\Theta) + \lambda_n \cR(\Theta)$ could also be strongly convex. In that case, the run-time to obtain an $\epsilon$-optimal solution of $\hThetan$ would scale as $O(\log(1/\epsilon))$.

\begin{remark}
	While the optimization in \eqref{eq:GRISMe_unconstrained}  is a convex minimization problem, computing the loss function as well as its gradient requires centering of the natural statistics. If the natural statistics are polynomials or trigonometric, centering them should be relatively straightforward (since the integrals would have closed-form expressions). In other cases, centering them may not be polynomial-time and one might require an assumption of computationally efficient sampling or that obtaining approximately random samples of $\rvbx$ is computationally efficient as in \cite{DiakonikolasKSS2021}.
\end{remark}

\section{Analysis and Main results}
\label{sec:main results}
In this section, we provide our analysis and main results. First, we focus on the connection between our method and the MLE of $\DensityDifferencefun$. Next, we draw the connections between our method and the Bregman score (i.e., score-based methods) as well as the {surrogate likelihood} (i.e., non-parametric density estimation).
Then, we establish consistency and asymptotic normality of our estimator. Finally, we provide non-asymptotic finite sample guarantees to recover $\ThetaStar$.
\subsection{Connection with MLE of $\DensityDifferencefun$}
First, we establish a connection between the population version of the loss function in \eqref{eq:sampleGISMe} (denoted by $\cL(\Theta)$) and the KL-divergence of the uniform density on $\cX$ with respect to $\DensityDifference$. Then, using \textit{minimality} of the exponential family, we show that this KL-divergence as well as $\cL(\Theta)$ are minimized if and only if $\Theta = \ThetaStar$. This shows that the proposed loss function is proper and provides an intuitive justification for the estimator in \eqref{eq:GRISMe}. 

For any $\Theta \in \parameterSet$, we have
$$\cL(\Theta)  = \Expectation \Big[\exp\big( -\big\langle \big\langle \Theta, \varPhi(\rvbx) \big\rangle \big\rangle \big)\Big].$$
The following result shows that the population version of the estimator in \eqref{eq:GRISMe} is equivalent to the maximum likelihood estimator of $\DensityDifference$. 
\begin{restatable}{theorem}{theoremKLD}\label{theorem:GRISMe-KLD}
	With $\infdiv{\cdot}{\cdot}$ representing the KL-divergence, 
	\begin{align}\label{eq:thm.1}
	\argmin_{\Theta \in \parameterSet} \cL(\Theta) =
	\argmin_{\Theta \in \parameterSet}
	\infdiv{\Uniformfun}{\DensityDifferencefun}. 
	\end{align}
	Further, the true parameter $\ThetaStar$ is the unique minimizer of $\cL(\Theta)$.
\end{restatable}
\begin{proof}[Proof of Theorem \ref{theorem:GRISMe-KLD}]
	First, we express $\DensityDifferencefun$ in terms of $\cL(\Theta)$. We have
	\begin{align}
	\DensityDifference & = 
	\frac{\exp\big(  \big\langle  \big\langle \ThetaStar - \Theta, \Phi(\svbx) \big\rangle \big\rangle \big)}{\int_{\svby \in \cX} \exp\big(  \big\langle  \big\langle \ThetaStar - \Theta, \Phi(\svby) \big\rangle \big\rangle \big) d\svby} \\
	& \stackrel{(a)}{=}  \frac{\exp\big(  \big\langle  \big\langle \ThetaStar - \Theta, \varPhi(\svbx) \big\rangle \big\rangle \big)}{\int_{\svby \in \cX} \exp\big(  \big\langle  \big\langle \ThetaStar - \Theta, \varPhi(\svby) \big\rangle \big\rangle \big) d\svby} \\
	& \stackrel{(b)}{=} \frac{\DensityXTrue \exp\big(  -\big\langle \big\langle \Theta, \varPhi(\svbx) \big\rangle \big\rangle\big) }{\int_{\svby \in \cX}   \DensityXTrue \exp\big(  -\big\langle \big\langle \Theta, \varPhi(\svby) \big\rangle \big\rangle\big) d\svby}\\
	& \stackrel{(c)}{=} \frac{\DensityXTrue \exp\big(  -\big\langle \big\langle \Theta, \varPhi(\svbx) \big\rangle \big\rangle\big) }{\cL(\Theta)},  \label{eq:re-express difference density}
	\end{align}
	where $(a)$ follows because $\Expectation_{\Uniform} [\Phi(\rvbx)]$ is a constant, $(b)$ follows by dividing the numerator and the denominator by the constant $\int_{\svby \in \cX} \exp\big(  \big\langle \big\langle  \ThetaStar, \varPhi(\svby) \big\rangle \big\rangle \big) d\svby$ and using the definition of $\DensityXTrue$, and
	$(c)$ follows from definition of $\cL(\Theta)$. Now, we simplify the KL-divergence between $\Uniformfun$ and $\DensityDifferencefun$. 
	\begin{align}
	\infdiv{\Uniformfun}{\DensityDifferencefun}  & \stackrel{(a)}{=} \Expectation_{\Uniform} \bigg[ \log\bigg( \dfrac{\Uniform(\cdot) \cL(\Theta) }{\DensityXTruefun \exp\big(  -\big\langle \big\langle \Theta, \varPhi(\cdot) \big\rangle \big\rangle \big) }\bigg) \bigg] \\
	& \stackrel{(b)}{=} \Expectation_{\Uniform} \bigg[ \log\bigg( \dfrac{\Uniform(\cdot)  }{\DensityXTruefun  }\bigg) \bigg] + \Expectation_{\Uniform} \Big[ \Big\langle \Big\langle  \Theta, \varPhi(\cdot) \Big\rangle \Big\rangle \Big] +  \log \cL(\Theta) \\
	& \stackrel{(c)}{=}  \Expectation_{\Uniform} \bigg[ \log\bigg( \dfrac{\Uniform(\cdot)  }{\DensityXTruefun  }\bigg) \bigg] +
	\Big\langle \Big\langle  \Theta, \Expectation_{\Uniform} [ \varPhi(\cdot) ] \Big\rangle \Big\rangle +  \log \cL(\Theta) \\
	& \stackrel{(d)}{=} \Expectation_{\Uniform} \bigg[ \log\bigg( \dfrac{\Uniform(\cdot)  }{\DensityXTruefun  }\bigg) \bigg] + \log \cL(\Theta),
	\end{align}
	where $(a)$ follows from \eqref{eq:re-express difference density} and the definition of KL-divergence, $(b)$ follows because $\log(abc) = \log a + \log b + \log c$ and $\cL(\Theta)$ is a constant, $(c)$ follows from the linearity of the expectation and $(d)$ follows because $\Expectation_{\Uniform} [ \varPhi(\rvbx) ] = 0$ from Definition \ref{def:css}.
	Observing that the first term in the above equation is not dependent on $\Theta$, we can write
	\begin{align}
	\argmin_{\Theta \in \parameterSet}
	\infdiv{\Uniformfun}{\DensityDifferencefun} %
	& = \argmin_{\Theta \in \parameterSet} \log \cL(\Theta) \\
	& \stackrel{(a)}{=} \argmin_{\Theta \in \parameterSet}  \cL(\Theta),
	\end{align}
	where $(a)$ follows because $\log$ is a monotonic function. Further, the KL-divergence between $\Uniformfun$ and $\DensityDifferencefun$ is uniqely minimized when $\Uniformfun = \DensityDifferencefun$. Recall that the natural statistic are such that the exponential family is minimal. Therefore, $\Uniformfun = \DensityDifferencefun$ if and only if $\Theta = \ThetaStar$. Thus, $\ThetaStar \in \argmin_{\Theta \in \parameterSet}  \cL(\Theta)$,
	and it is a unique minimizer of $\cL(\Theta)$.
\end{proof}
\subsection{Connections to Bregman score}
\label{subsec:bregman_score}
We now show that the estimator in \eqref{eq:GRISMe} is an \textit{optimal score estimator} (to be defined). More specifically, we show that $\exp\big( -\big\langle \big\langle \Theta, \varPhi(\svbx) \big\rangle  \big\rangle \big)$ is equal to a Bregman scoring rule (which we make precise below).

Let $\qThetafun$ be a measurable function parameterized by $\Theta \in \parameterSet$. Let $\rvbx$ be a random variable whose distribution is proportional to $\qThetafuntrue$ for $\ThetaStar \in \parameterSet$. A scoring rule $S(\svbx, \qThetaX)$ \citep{GneitingR2007} is a numerical score assigned to a realization $\svbx$ of $\rvbx$ and it measures the quality of the predictive function $\qThetafun$. Given $n$ samples $\svbx^{(1)}, \cdots, \svbx^{(n)}$ of $\rvbx$, an \textit{optimal score estimator} $\hTheta_{n, S}$ (of $\ThetaStar$) associated with the scoring rule $S$ is
\begin{align}
\hTheta_{n, S} \in \argmin_{\Theta \in \parameterSet} \frac{1}{n} \sum_{t=1}^{n} S(\svbx^{(t)}, \qThetaXt). \label{eq:optimal_score_estimator}
\end{align}
Further, a scoring rule is said to be \textit{a proper scoring rule} if the expected score $\Expectation[S(\rvbx, \qThetaRVX)]$ is uniquely minimized when $\Theta = \ThetaStar$. 

The (separable) Bregman scoring rule \citep{GrunwaldD2004} associated with a convex and differentiable function $\psi : \Reals^+ \rightarrow \Reals$ and a baseline measure $\rho$ is given by
\begin{align}
\hspace{-6mm} S_{\psi, \rho}(\svbx, \qThetaX) =  -\psi'(\qThetaX) -  \Expectation_{\rho} [ \psi(\qThetaX) - \qThetaX \psi'(\qThetaX) ]. \label{eq:Bregman}
\end{align}
\begin{restatable}{proposition}{propBregman}\label{prop:bregman_score}
	Let $\psi(\cdot) = -\log (\cdot)$, $\rho(\cdot) = \Uniformfun$, and $\qThetafun = \exp\big( \big\langle  \big\langle \Theta, \varPhi(\cdot) \big\rangle \big\rangle \big)$. Then,
	\begin{align}
	S_{\psi, \rho}(\cdot, \qThetafun) & = \exp\Big( -\Big\langle \Big\langle \Theta, \varPhi(\cdot) \Big\rangle \Big\rangle \Big) - 1 ~ \mbox{and}  \\
	\hTheta_{n, S_{\psi, \rho}} & =  \hTheta_n.
	\end{align}
	Further, the scoring rule $S_{\psi, \rho}(\cdot, \qThetafun)$ is \textit{proper}.
\end{restatable}

\begin{proof}[Proof of Proposition \ref{prop:bregman_score}]
	Choosing $\psi(\cdot) = - \log (\cdot)$, the Bregman scoring rule in \eqref{eq:Bregman} simplifies to
	\begin{align}
	S_{\psi, \rho}(\svbx, \qThetaX)  = 1 / \qThetaX + \Expectation_{\rho} [ \log(\qThetaX) - 1 ] .
	\end{align}
	Now, letting $\qThetafun = \exp\Big(  \Big\langle \Big\langle \Theta, \varPhi(\cdot) \Big\rangle \Big\rangle \Big)$ and $\rho(\cdot) = \Uniformfun$, we have
	\begin{align}
	S_{\psi, \rho}(\svbx, \qThetaX) & = \exp\Big(  - \Big\langle \Big\langle \Theta, \varPhi(\svbx) \Big\rangle \Big\rangle \Big) + \Expectation_{\Uniform} \Big[ \Big\langle \Big\langle \Theta, \varPhi(\rvbx) \Big\rangle \Big\rangle  \Big] - 1\\
	& \stackrel{(a)}{=} \exp\Big(  - \Big\langle \Big\langle  \Theta, \varPhi(\svbx) \Big\rangle \Big\rangle  \Big) + \Big\langle \Big\langle  \Theta, \Expectation_{\Uniform} [ \varPhi(\rvbx) ] \Big\rangle \Big\rangle  - 1 \\
	& \stackrel{(b)}{=} \exp\Big(  - \Big\langle \Big\langle  \Theta, \varPhi(\svbx) \Big\rangle \Big\rangle  \Big) - 1, \label{eq:Bregman_equi}
	\end{align}
	where $(a)$ follows because the integral of a sum is equal to the sum of the integrals and $(b)$ follows because $\Expectation_{\Uniform} [ \varPhi(\rvbx) ] = 0$ from Definition \ref{def:css}. The equivalence between $\hTheta_{n, S_{\psi, \rho}}$ and $\hTheta_n$ follows by plugging \eqref{eq:Bregman_equi} in \eqref{eq:optimal_score_estimator}. Further, from Theorem \ref{theorem:GRISMe-KLD}, $\ThetaStar$ is the unique minimizer of $\Expectation\big[\exp\big(  - \big\langle \big\langle \Theta, \varPhi(\svbx) \big\rangle \big\rangle \big) \big]$. Therefore, the scoring rule $S_{\psi, \rho}(\cdot, \qThetafun)$ is \textit{proper}.
\end{proof}

We note that by letting $\qThetafun = \exp\big( \big\langle \big\langle \Theta, \varPhi(\cdot) \big\rangle \big\rangle \big)$, we inherently make use of the extension of Bregman scoring rule beyond the probability simplex \citep{PainskyW2019}. Also, the fact that $S_{\psi, \rho}(\cdot, \qThetafun)$ is \textit{proper} scoring rule should not be surprising since all Bregman scoring rules are \textit{proper} \citep{GneitingR2007}.

\subsection{Connections to non-parametric density estimation}
\label{subsec:non_para_density_estimation}
We now show that the loss function proposed in \eqref{eq:GRISMe} is an instance of the {surrogate} likelihood proposed by \cite{JeonL2006}. As described in Section \ref{subsec_related_work}, to bypass the computational hardness of the maximum likelihood estimation, \cite{JeonL2006} proposed using the following {surrogate} likelihood for learning non-parametric densities of the form $f_{\rvbx}(\svbx) \propto \exp(\eta(\svbx))$:
\begin{align}
\cJ_n(\eta) =  \frac{1}{n} \sum_{t = 1}^{n} \exp(-\eta(\svbx^{(t)})) +  \int_{\svbx \in \cX} \rho(\svbx) \eta(\svbx) d\svbx, \label{eq:jeonlin}
\end{align}
where $\rho(\cdot)$ is some fixed known density with the same support $\cX$ as the unknown density $f_{\rvbx}(\svbx)$. The following proposition shows that the loss function $\cL_{n}(\Theta)$ is equivalent to $\cJ_n(\eta)$ for a specific choice of $\rho(\cdot)$.

\begin{restatable}{proposition}{propsurrogate}\label{prop:surrogate}
	Let $\rho(\cdot) = \Uniformfun$ and $\eta(\cdot) = \big\langle \big\langle \Theta, \varPhi(\cdot) \big\rangle \big\rangle $. Then, $\cJ_n(\eta) = \cL_{n}(\Theta) $. 
\end{restatable}
\begin{proof}[Proof of Proposition \ref{prop:surrogate}]
	Plugging in $\rho(\cdot) = \Uniformfun$ and $\eta(\cdot) = \big\langle \big\langle \Theta, \varPhi(\cdot) \big\rangle \big\rangle  $ in \eqref{eq:jeonlin}, we have
	\begin{align}
	\cJ_n(\eta) & =  \frac{1}{n} \sum_{t = 1}^{n} \exp\Big(-\Big\langle \Big\langle \Theta, \varPhi(\svbx^{(t)}) \Big\rangle \Big\rangle \Big)  +  \int_{\svbx \in \cX} \Uniform(\svbx) \Big\langle \Big\langle \Theta, \varPhi(\svbx) \Big\rangle \Big\rangle d\svbx \\
	& \stackrel{(a)}{=}  \cL_{n}(\Theta)  + \Big\langle \Big\langle \Theta, \int_{\svbx \in \cX} \Uniform(\svbx) \varPhi(\svbx) d\svbx \Big\rangle \Big\rangle  \\
	& \stackrel{(b)}{=}  \cL_{n}(\Theta), 
	\end{align}
	where $(a)$ follows from \eqref{eq:sampleGISMe} and because the integral of a sum is equal to the sum of the integrals and $(b)$ follows because $\Expectation_{\Uniform} [ \varPhi(\rvbx) ] = 0$ from Definition \ref{def:css}.
\end{proof}

\subsection{Consistency and Normality}
\label{subsec:consistency_and_normality}
We establish consistency and asymptotic normality of the proposed estimator $\hTheta_n$ by invoking the asymptotic theory of M-estimation. We emphasize that, from Theorem \ref{theorem:GRISMe-KLD}, the population version of $\hTheta_n$ is equivalent to the maximum likelihood estimate of $\DensityDifferencefun$ and not $\DensityXfun$. Moreover, there is no clear connection between $\hTheta_n$ and the finite sample maximum likelihood estimate of $\DensityXfun$ or $\DensityDifferencefun$. Therefore, we cannot invoke the asymptotic theory of MLE to show consistency and asymptotic normality of $\hTheta_n$.

Let $A(\ThetaStar)$ denote the covariance matrix of $\vect\big(\varPhi(\rvbx)\exp\big( -\big\langle \big\langle \ThetaStar, \varPhi(\rvbx) \big\rangle \big\rangle \big)\big)$. Let $B(\ThetaStar)$ denote the cross-covariance matrix of $\vect(\varPhi(\rvbx))$ and $\vect(\varPhi(\rvbx) \exp\big( -\big\langle \big\langle \ThetaStar, \varPhi(\rvbx) \big\rangle \big\rangle \big))$. Let ${\cal N}(\bm{\mu}, \bm{\Sigma})$ represent the multi-variate Gaussian distribution with mean vector $\bm{\mu}$ and covariance matrix $\bm{\Sigma}$.
\begin{restatable}{theorem}{theoremconsistencynormality}
	\label{thm:consistency_normality}
	Let Assumptions \ref{bounds_parameter}, \ref{bounds_statistics}, and \ref{bounds_statistics_maximum} be satisfied.
	Let $\hThetan$ be a solution of \eqref{eq:GRISMe}. Then, as $n\to \infty$, $\hThetan \stackrel{p}{\to} \ThetaStar$. Further, assuming  $\ThetaStar \in \text{interior}(\parameterSet)$ and $B(\ThetaStar)$ is invertible, we have
	$\sqrt{n} \times \vect( \hThetan - \ThetaStar ) \stackrel{d}{\to} {\cal N}(\vect(\boldsymbol{0}),B(\ThetaStar)^{-1} A(\ThetaStar) B(\ThetaStar)^{-1})$.
\end{restatable}
The proof of Theorem \ref{thm:consistency_normality} can be found in Appendix \ref{appendix:proof of thm:consistency_normality}. The proof is based on two key observations : (a) $\hThetan$ is an $M$-estimator (which follows from \eqref{eq:sampleGISMe}) and (b) $\cL(\Theta)$ is uniquely minimized at $\ThetaStar$ (which follows from Theorem \ref{theorem:GRISMe-KLD}).

\subsection{Finite Sample Guarantees}
To provide the non-asymptotic guarantees for recovering $\ThetaStar$, we require the following assumption on the smallest eigenvalue of the autocorrelation matrix of $\vect(\varPhi(\rvbx))$.
\begin{assumption}\label{lambdamin}(Positive eigenvalue of the autocorrelation matrix of $\varPhi$.)
	Let $\lambdaMin$ denote the minimum eigenvalue of $\Expectation_{\rvbx}[\vect(\varPhi(\rvbx)) \vect(\varPhi(\rvbx))^T]$. We assume $\lambdaMin$ is strictly positive i.e., $\lambdaMin > 0$.
\end{assumption}

Theorem \ref{thm:finite_sample} below shows that, with enough samples, the $\epsilon$-optimal solution of $\hTheta_n$ 
is close to the true natural parameter in the tensor norm with high probability. 
\begin{restatable}{theorem}{theoremfinite}\label{thm:finite_sample}
	Let Assumptions \ref{bounds_parameter}, \ref{bounds_statistics}, \ref{bounds_statistics_maximum}, and \ref{lambdamin} be satisfied. Define
	\begin{align}
	\gamma(\dimOne, \dimTwo) & = \max_{\bM \in 4\parameterSet \setminus \{\boldsymbol{0}\}} \frac{\|\bM\|_{1,1}}{\|\bM\|_{\mathrm{F}} }, \\
	 g(\dimOne, \dimTwo) & = \max_{\bM \in \Reals^{\dimOne \times \dimTwo} \setminus \{\boldsymbol{0}\}} \frac{\cR^*(\bM)}{\|\bM\|_{\max} }, ~~ \text{and} \\ 
	\Psi(\dimOne, \dimTwo) & = \max_{\bM \in \parameterSet \setminus \{\boldsymbol{0}\}} \frac{\cR(\bM)}{\|\bM\|_{\mathrm{F}}}. 
	\label{eq_defn_thm}
	\end{align}
	Let $\hThetan$ be a minimizer of the optimization in \eqref{eq:GRISMe_unconstrained}. Then, for any $\alpha >0$ and $\delta \in (0,1)$, we have $\|\hThetan - \ThetaStar\|_{\mathrm{F}} \leq \alpha$
	with probability at least $1-\delta$ as long as
	\begin{align}
	n = \!\Omega\bigg(\!\frac{ \max\big\{  \gamma^4(\dimOne, \dimTwo), g^2(\dimOne, \dimTwo) \Psi^2(\dimOne, \dimTwo) \big\} }{\alpha^2 \lambdaMin^2}\!\log\!\Big(\!\frac{4\dimOne^2 \dimTwo^2}{\delta}\Big) \!\!\bigg).
	\label{eq:sample_complexity}
	\end{align}
\end{restatable}
The proof of Theorem \ref{thm:finite_sample} can be found in Appendix \ref{appendix_proof_finite_sample}. The proof builds on techniques in \cite{negahban2012unified, VuffrayMLC2016, VuffrayML2019, ShahSW2021} and is based on two key properties of the loss function $\cL_n(\Theta)$: (a) with enough samples, the loss function $\cL_n(\Theta)$ naturally obeys the restricted strong convexity with high probability (see Proposition \ref{prop:rsc_GISMe}) and (b) with enough samples, $\| \nabla \cL_n(\ThetaStar) \|_{\max}$ is bounded with high probability (see Proposition \ref{prop:gradient-concentration-GISMe}.). See the proof for the dependence of the sample complexity on $\br, \bd$ and $\phiMax$ as well as for the choice of the regularization penalty $\lambda_n$. 

The sample complexity in \eqref{eq:sample_complexity} depends on the functions $\gamma(\dimOne, \dimTwo)$, $g(\dimOne, \dimTwo)$, and $\Psi(\dimOne, \dimTwo)$. Now, we provide bounds on these functions. First, it is easy to see that $\gamma(\dimOne, \dimTwo) \leq \sqrt{\dimOne \dimTwo}$ for any $\parameterSet$. Second, in Appendix \ref{appendix:dual_norm}, we show that 
\begin{enumerate}
	\item $g(\dimOne, \dimTwo) \leq \dimOne^{\frac{1}{p}}  \dimTwo^{\frac{1}{q}}$ whenever the dual norm $\cR^*$ is either the entry-wise $L_{p,q}$ norm $(p,q \geq 1)$,
	\item $g(\dimOne, \dimTwo) \leq \sqrt{\min\{\dimOne, \dimTwo\}\dimOne \dimTwo}$ whenever the dual norm $\cR^*$ is the Schatten $p$-norm $(p \geq 1)$, and 
	\item $g(\dimOne, \dimTwo) \leq \dimOne^{\frac{1}{p}}  \dimTwo^{1-\frac{1}{p}}$ whenever the dual norm $\cR^*$ is the operator $p-$norm $(p \geq 1)$. 
\end{enumerate}
Lastly, it is easy to see that $\Psi(\dimOne, \dimTwo) \leq 1$ whenever the norm $\cR$ is either the entry-wise $L_{p,p}$ norm $(p \geq 2)$ or the Schatten $p$-norm $(p \geq 2)$. In these scenarios, the sample complexity in \eqref{eq:sample_complexity} can be simplified to $n = \Omega\big(\frac{ \dimOne^2 \dimTwo^2 }{\alpha^2}\log\big(\frac{4\dimOne^2 \dimTwo^2}{\delta}\big) \big)$. For other norms, bounds (not necessarily tight) similar to $g(\dimOne, \dimTwo)$ can also be obtained on $\Psi(\dimOne, \dimTwo)$ by noting that $\|\bM\|_{\max} \leq \|\bM\|_{\mathrm{F}}$. In these scenarios, the sample complexity in \eqref{eq:sample_complexity} can be simplified to $n = \Omega\big( \frac{\mathrm{poly} (\dimOne \dimTwo)}{\alpha^2} \log \big(\frac{1}{\delta}\big)\big)$.

The following corollary (stated without proof) provides a formal version of our finite sample guarantees for the examples in Section \ref{sec:prob_formulation}, i.e., when the underlying norm $\cR$ is either the maximum norm, the Frobenius norm, or the nuclear norm. We note that Theorem \ref{thm:finite_sample} could be specialized for other norms as well.
\begin{corollary}\label{coro_main}
	Let Assumptions \ref{bounds_parameter}, \ref{bounds_statistics}, \ref{bounds_statistics_maximum}, and \ref{lambdamin} be satisfied. Suppose $\cR(\Theta) = \|\Theta\|_{\max}$, $\cR(\Theta) = \|\Theta\|_{\mathrm{F}}$, or $\cR(\Theta) = \|\Theta\|_{\star}$. Let $\hThetan$ be a minimizer of the optimization in \eqref{eq:GRISMe_unconstrained}. Then, for any $\alpha >0$ and $\delta \in (0,1)$, we have $\|\hThetan - \ThetaStar\|_{\mathrm{F}} \leq \alpha$ with probability at least $1-\delta$ as long as
	\begin{align}
	n = \Omega\bigg(\frac{ \dimOne^2 \dimTwo^2 }{\alpha^2}\log\Big(\frac{4\dimOne^2 \dimTwo^2}{\delta}\Big) \bigg).
	\end{align}
\end{corollary}

\begin{remark}\label{remark_mrf}
The result in Theorem \ref{thm:finite_sample} could also be specialized for learning node-wise-sparse pairwise MRFs. Under this setting, as is typical, the machinery developed could be applied to the node-conditional distribution of $\rvx_i$, i.e., the conditional distribution in \eqref{eq:conditional}, for every $i \in [p]$, one at a time. Then, with $\dimOne = p$ and $\dimTwo = 1$, the parameter set $\parameterSet$ is defined as the set of all $r$-sparse $p$-dimensional vectors where $r$ is assumed to be a constant. To enforce the sparsity, $\cR$ is chosen to be the $\ell_1$ norm resulting in $\cR^*$ being equal to the maximum norm. Then, it is easy to see that $\gamma(p)$ and $\Psi(p)$ are $O(\sqrt{r})$, and $g(p) = 1$. As a result, the sample complexity in \eqref{eq:sample_complexity} can be simplified to $n = \Omega\big( \frac{1}{\alpha^2} \log \big(\frac{p}{\sqrt{\delta}}\big)\big)$. The logarithmic dependence on $p$ is consistent with the literature on binary, discrete, Gaussian as well as continuous MRFs \citep{VuffrayMLC2016,VuffrayML2019,DaskalakisGTZ2018,ShahSW2021}. The $1/\alpha^2$ dependence on the error tolerance is consistent with the literature on binary and Gaussian MRFs \citep{VuffrayMLC2016,DaskalakisGTZ2018} and is an improvement over the literature on discrete and continuous MRFs \citep{VuffrayML2019,ShahSW2021}.
\end{remark}

\section{Experiments}
\label{sec:experiments}
In this section, we demonstrate our experimental findings on the three examples from Section \ref{sec:prob_formulation} using synthetic data. We consider the Frobenius norm constraint in the first example, the maximum norm constraint in the second example, and the nuclear norm constraint (which is a relaxation of the low-rank constraint) in the third example.

\subsection{Frobenius norm constraint}
\label{subsec:frob_norm_exp}
We consider the random vector $\rvbx$ belonging to $\cX$ for two different choices of $\cX$: (a) $\cX = \cB_1(b)$ and (b) $\cX =  [-b,b]^p$ for some $b \in \Reals_+$. We let $\dimOne = \dimTwo = p$ and let the natural statistics be polynomials of degree two i.e., $\Phi_{ij} = x_ix_j$ for all $ i \in [p], j \in [p]$. Summarizing, the family of distributions considered is as follows:
\begin{align}
\DensityX \propto \exp\Big(  \sum_{i,j \in [p]} \Theta_{ij} x_i x_j\Big), \label{eq:densityTG}
\end{align}
where $\rvbx \in \cX$ and $\|\Theta\|_{\mathrm{F}} \leq r$ for some constant $r$. As in Section \ref{sec:prob_formulation}, let $\DensityXTrue$ denote the true distribution of $\rvbx$ and $\ThetaStar$ denote the true natural parameter of interest such that $\|\ThetaStar\|_{\mathrm{F}} \leq r$. Further, we have $\parameterSet = \{ \Theta \in \Reals^{p \times p} : \|\Theta\|_{\mathrm{F}} \leq r\}$.

For our first choice of $\cX$, i.e., $\cX = \cB_1(b)$, the  family of distributions in \eqref{eq:densityTG} satisfies Assumption \ref{bounds_statistics} with $d = (1+b)^2$ and Assumption \ref{bounds_statistics_maximum} with $\phiMax = \max\{1,b^2\}$. For our second choice of $\cX$, i.e., $\cX =  [-b,b]^p$, the  family of distributions in \eqref{eq:densityTG} satisfies Assumption  \ref{bounds_statistics_maximum} with $\phiMax = \max\{1,b^2\}$. In contrast, the constant $\bd$ in Assumption \ref{bounds_statistics} scales quadratically in $p$. As a result, the analytical bound on the  sample complexity from Corollary \ref{coro_main} suggests an exponential dependence on $p$ (see equation \eqref{eq:sample_comp_dependence} in the proof of Theorem \ref{thm:finite_sample} for the dependence on $d$). However, we see that the empirical bound on the sample complexity scales only polynomially in $p$, i.e., it is in agreement with Corollary \ref{coro_main}, suggesting that Assumption \ref{bounds_statistics} may not be a strict requirement in practice. For brevity, we only provide results with $\cX = \cB_1(b)$. The results with $\cX = [-b,b]^p$ are analogous.

We choose $b = 1$ and let the true natural parameter $\ThetaStar$ be as follows: 
\begin{align}
\ThetaStar_{ij}  = 
     \begin{cases}
        \frac{1}{\sqrt{p}} ~~ \text{if} ~~ i = 1, ~~ \text{or} ~~ j = 1, ~~ \text{or} ~~ i = j, \\
         0 ~~ \text{otherwise.}
     \end{cases}
 \label{eq:theta_frob}
\end{align}
This choice ensures that $\|\ThetaStar\|_{\mathrm{F}} \leq r$, i.e., $r = 1$. Further, the choice also ensures that the maximum node-degree in the underlying undirected graphical model is $p$ and the total number of edges scale linearly with $p$. This is easy to see as the undirected graph is a star graph with $\rvx_1$ as the center of the star. We note that this is in contrast with the literature on node-wise-sparse pairwise MRFs (see Section \ref{subsec_related_work}) where the total number of edges scale linearly with $p$ but the maximum node-degree does not depend on $p$. Therefore, the techniques developed to learn the parameters of such MRFs are not useful here. We also note that the $1/\sqrt{p}$ scaling in \eqref{eq:theta_frob} is consistent with the  Sherington-Kirkpatrick model \citep{sherrington1975solvable}. Finally, to draw high-quality samples from \eqref{eq:densityTG}, we employ brute-force sampling using fine discretization with 100 bins per dimension. \\

 \begin{figure}[ht!]
	\centering
	\begin{tabular}{cc}
		\includegraphics[width=0.4\linewidth]{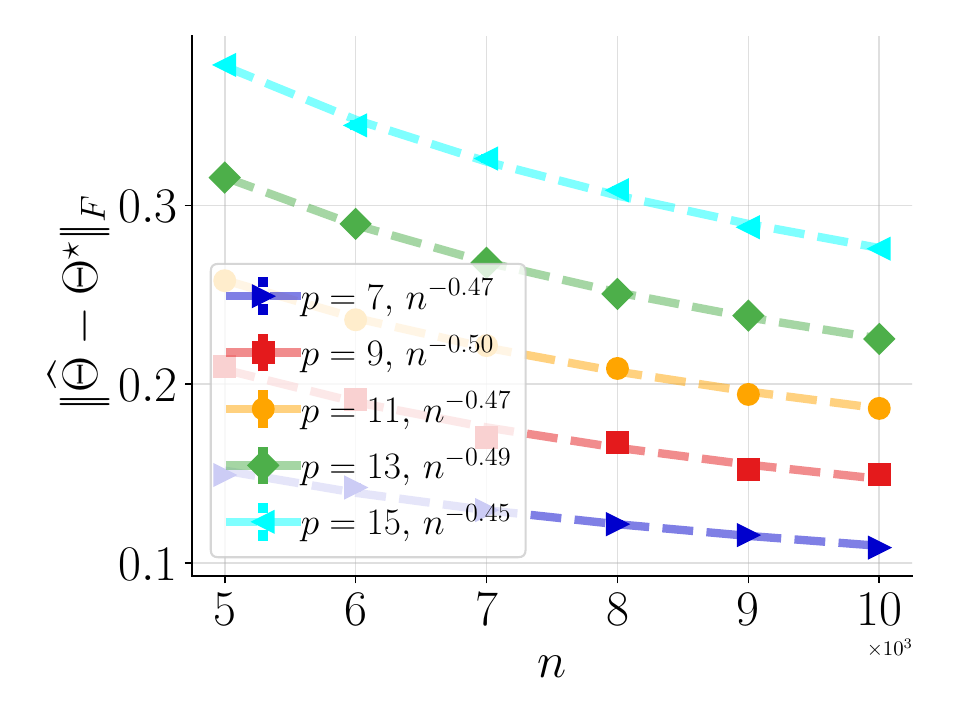} & \includegraphics[width=0.4\linewidth]{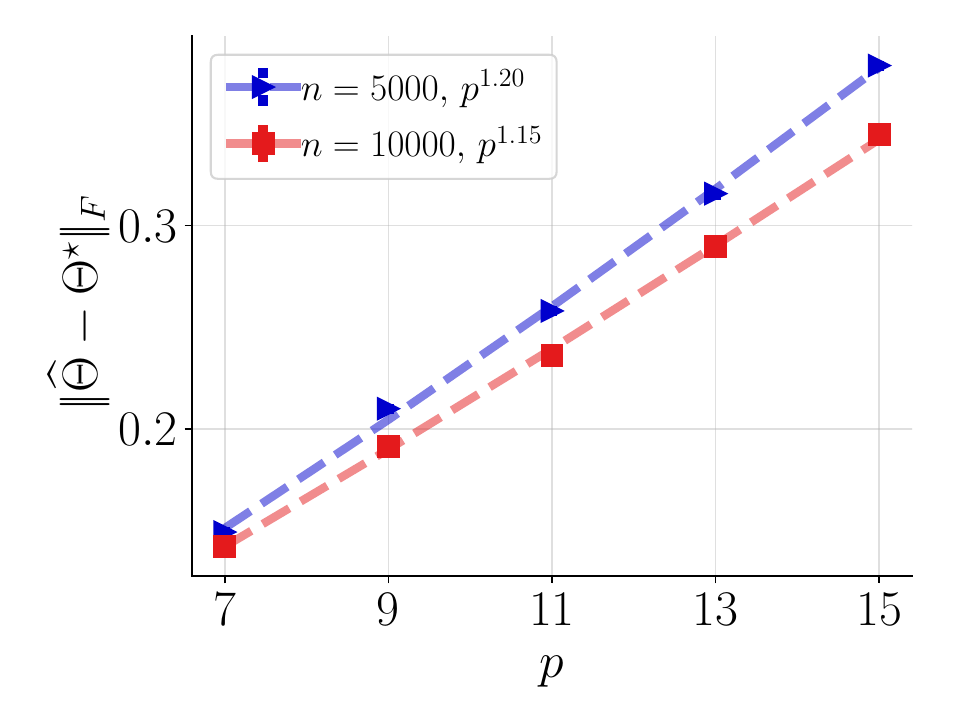} \\
		(a) & (b) \\ 
		\includegraphics[width=0.4\linewidth]{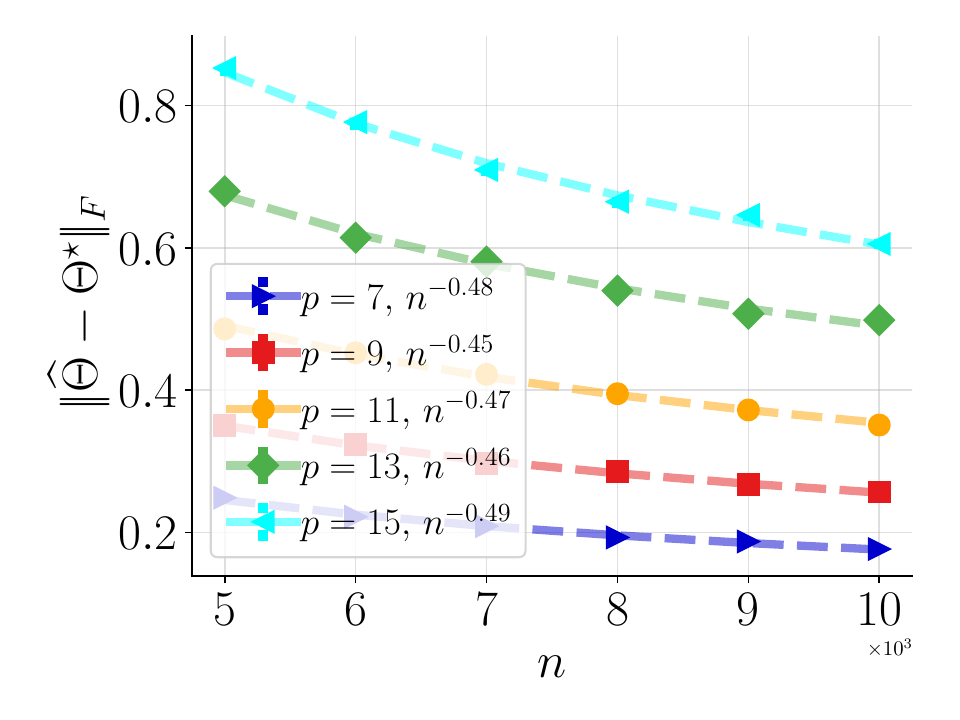} & \includegraphics[width=0.4\linewidth]{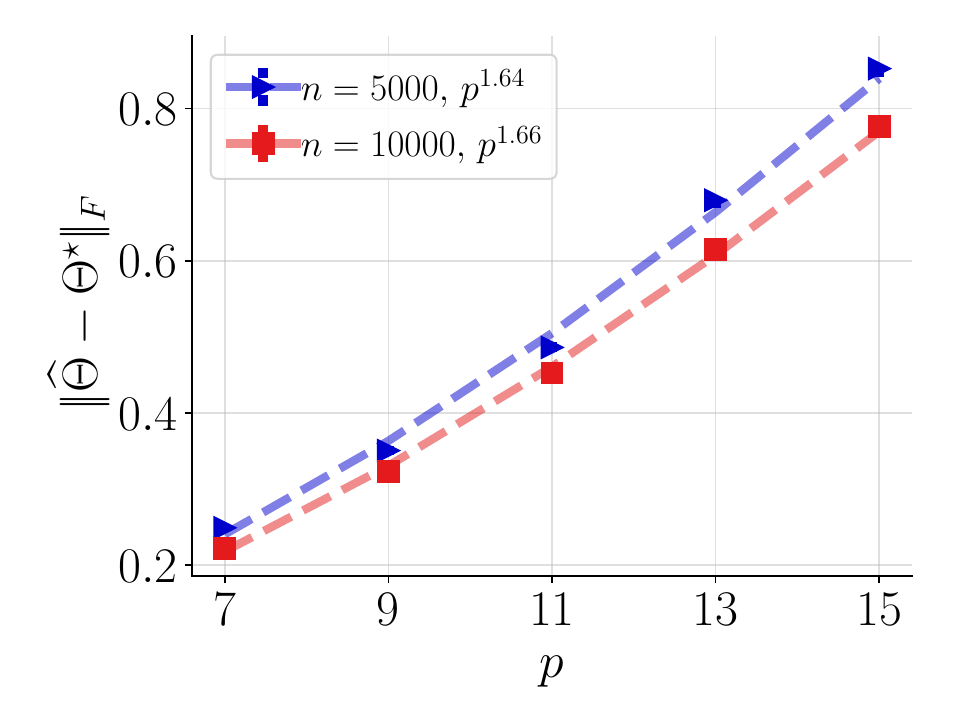} \\
		(c) & (d) \\ 
		\includegraphics[width=0.4\linewidth]{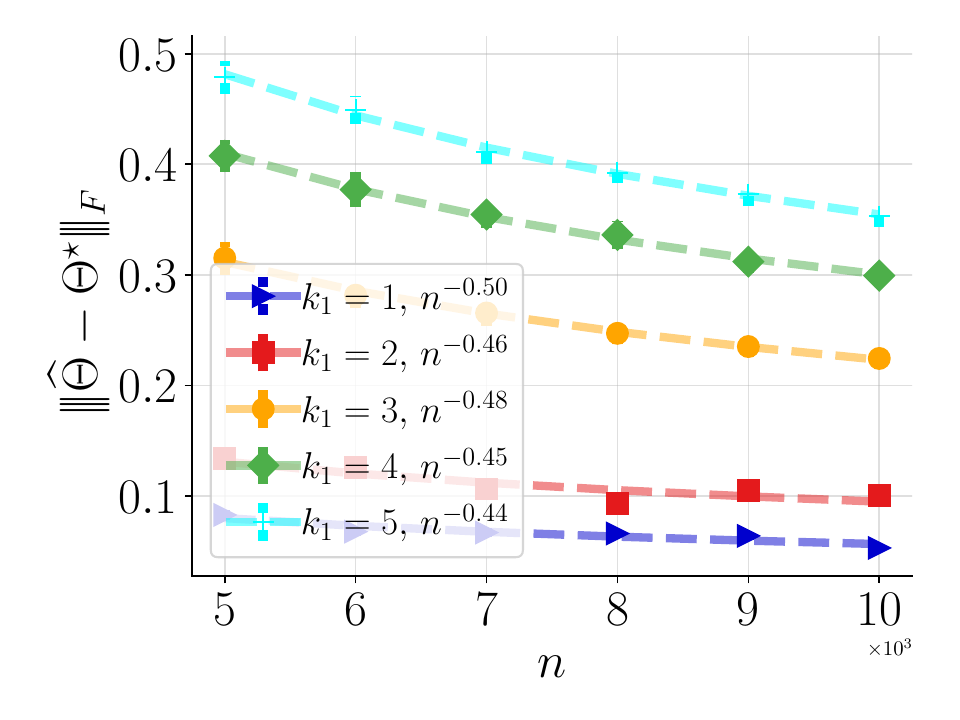} & \includegraphics[width=0.4\linewidth]{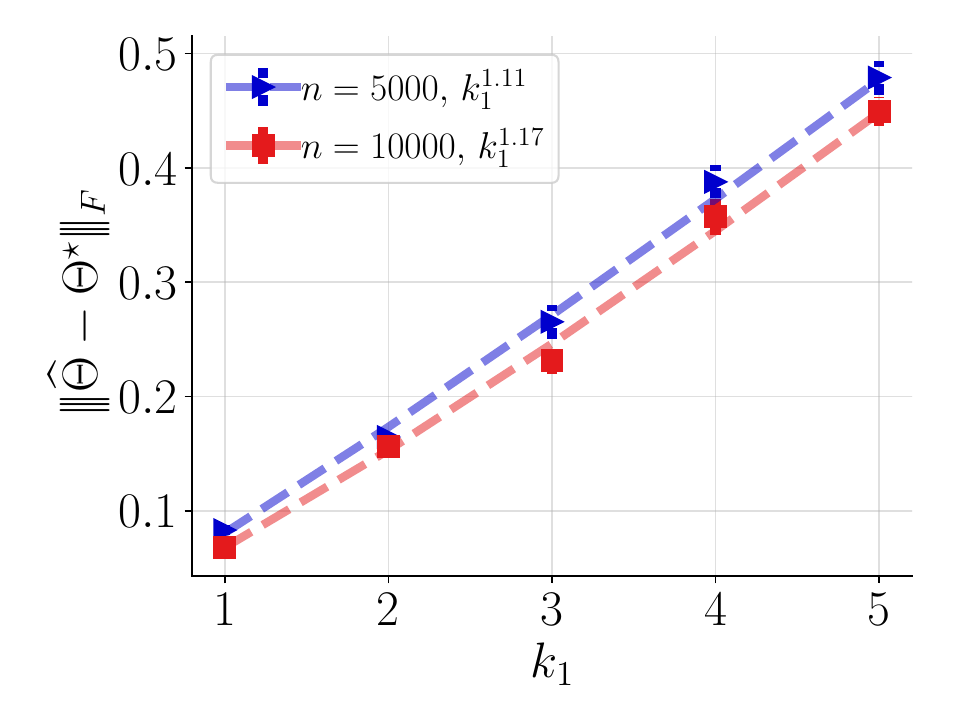} \\
		(e) & (f)
	\end{tabular}
	\caption{
		Error scaling for Frobenius norm constraint (in (a) and (b)), maximum norm constraint (in (c) and (d)), and the nuclear norm constraint (in (e) and (f)) with number of samples $n$ for various $p$ or $\dimOne$ (in (a), (c), (e)) and with number of parameters $p$ or $\dimOne$ for various $n$ (in (b), (d), (f)).}
	\label{fig:vs_p_n}
\end{figure}     

\noindent \textbf{Results.} In Figure \ref{fig:vs_p_n}(a), we plot the scaling of errors in our estimates for $\ThetaStar$, i.e., $\|\hThetan - \ThetaStar\|_{\mathrm{F}}$ as a function of number of sample $n$ for various $p$. Likewise, we present how the error scales as the dimension $p$ grows for various $n$ in Figure \ref{fig:vs_p_n}(b). We plot the averaged error across 100 independent trials along with $\pm 1$ standard error (the standard error is too small to be visible in our results). To help see the error scaling, we display the best linear fit (fitted on the log-log scale) and mention an empirical decay rate in the legend based on the slope of that fit, e.g., for a slope of $-0.47$ for estimating $\ThetaStar$ when $p = 7$, we report an empirical rate of $n^{-0.47}$ for the averaged error.

We observe that the error $\|\hThetan - \ThetaStar\|_{\mathrm{F}}$ admits a scaling of  between $n^{-0.50}$ and $n^{-0.45}$ for various $p$ and between $p^{1.15}$ and $p^{1.20}$. These empirical rates indicate a parametric error rate of $\sqrt{\frac{p^2 \log p}{n}}$ for $\|\hThetan - \ThetaStar\|_{\mathrm{F}}$. Our analysis from  Corollary \ref{coro_main} suggested the rate of $\sqrt{\frac{p^4 \log p}{n}}$. Thus, we see that the empirical dependence on $n$ is consistent with the theoretical whereas the dependence on $p = \sqrt{\dimOne \dimTwo}$ is an improvement.

\subsection{Maximum norm constraint}
\label{subsec:max_norm_exp}
We consider the same setup as in Section \ref{subsec:frob_norm_exp} except for the choices of $\ThetaStar$ and $\cX$. We 
let the true natural parameter $\ThetaStar$ be as follows:
\begin{align}
\ThetaStar_{ij} & = -0.1 - 0.4 \times \Indicator(i = j)  - 0.2 \times \Indicator(|i - j| = 1)
-0.1 \times \Indicator(|i - j| = 2). \label{eq:theta_truncated_gaussian}
\end{align}
This choice ensures that $\|\ThetaStar\|_{\max} \leq r$, i.e., $r = 0.5$. Further, the choice also ensures that the maximum node-degree in the underlying undirected graphical model is $p$ and the total number of edges scale quadractically with $p$. This is easy to see as the undirected graph is a complete graph as every entry of $\ThetaStar$ is non-zero. Further, we also note the choice of $\ThetaStar$ ensures that the inverse of $\ThetaStar$ is positive semi-definite. Therefore, the distribution of $\rvbx$ is equivalent to a Gaussian with mean equal to zero and inverse covariance equal to $\ThetaStar$ but the support truncated to $\cX$. Then, we use the \textit{tmvtnorm} package \citep{WilhelmM2010} to generate samples from \eqref{eq:densityTG} via rejection sampling, and choose $\cX = [-b,b]^p$ (with $b = 1$) for a higher acceptance probability.\\

\noindent \textbf{Results.} In Figure \ref{fig:vs_p_n}(c), we plot the scaling of errors in our estimates for $\ThetaStar$, i.e., $\|\hThetan - \ThetaStar\|_{\mathrm{F}}$ as a function of number of sample $n$ for various $p$. Likewise, we present how the error scales as the dimension $p$ grows for various $n$ in Figure \ref{fig:vs_p_n}(d). We observe that the error $\|\hThetan - \ThetaStar\|_{\mathrm{F}}$ admits a scaling of  between $n^{-0.49}$ and $n^{-0.45}$ for various $p$ and between $p^{1.64}$ and $p^{1.66}$. These empirical rates indicate a parametric error rate of $\sqrt{\frac{p^3 \log p}{n}}$ for $\|\hThetan - \ThetaStar\|_{\mathrm{F}}$. Our analysis from  Corollary \ref{coro_main} suggested the rate of $\sqrt{\frac{p^4 \log p}{n}}$. Thus, we see that the empirical dependence on $n$ is consistent with the theoretical wheares the dependence on $p = \sqrt{\dimOne \dimTwo}$ is an improvement.

\subsection{Nuclear norm constraint}
\label{subsec:low_rank_norm_exp}
As stated in Section \ref{subsec:examples}, we let $\cX = \cB_2(b)$. We consider the dimension $p = 2$ and vary the number of natural parameters $k = \dimOne \times \dimTwo$. We let the natural statistics be polynomials of varying degree i.e., $\Phi_{ij} = x_1^ix_2^j$ for all $i \in [\dimOne], j \in [\dimTwo]$. Summarizing, the family of distribution considered is as follows:
\begin{align}
\DensityX \propto \exp\Big(\sum_{i \in [\dimOne], j \in [\dimTwo]} \Theta_{i,j} x_1^i x_2^j \Big), \label{eq:densityLR1}
\end{align}
where $\rvbx \in \cX$ and $\| \Theta^{(1)} \|_{\star}= r$ for some constant $r$. As in Section \ref{sec:prob_formulation}, let $\DensityXTrue$ denote the true distribution of $\rvbx$ and $\ThetaStar$ denote the true natural parameter of interest such that $\|\ThetaStar\|_{\star} \leq r$. Further, we have $\parameterSet = \{ \Theta \in \Reals^{\dimOne \times \dimTwo} : \|\Theta\|_{\star} \leq r\}$.\\

In our experiments, we fix $\dimTwo = 2$ and vary $\dimOne$ from $1$ to $5$. For $\dimOne = 1$, we let the true natural parameter $\ThetaStar$ be as follows:
\begin{align}
\ThetaStar = \begin{bmatrix} 1 & 0.8 \end{bmatrix}.
\end{align}
This choice ensures that $\|\ThetaStar\|_{\star} \leq r$, i.e., $r = 1$. To ensure that $r = 1$ for $\dimOne > 1$, we let $\ThetaStar_{i,j} = \ThetaStar_{i-1,j} /2$ i.e., we let every row of $\ThetaStar$ to be a multiple of its first row. To draw high-quality samples from \eqref{eq:densityLR1}, we employ brute-force sampling using fine discretization with 100 bins per dimension.\\

\noindent \textbf{Results.} In Figure \ref{fig:vs_p_n}(e), we plot the scaling of errors in our estimates for $\ThetaStar$, i.e., $\|\hThetan - \ThetaStar\|_{\mathrm{F}}$ as a function of number of sample $n$ for various $\dimOne$. Likewise, we present how the error scales as the dimension $\dimOne$ grows for various $n$ in Figure \ref{fig:vs_p_n}(f). We observe that the error $\|\hThetan - \ThetaStar\|_{\mathrm{F}}$ admits a scaling of  between $n^{-0.50}$ and $n^{-0.44}$ for various $p$ and between $\dimOne^{1.11}$ and $\dimOne^{1.17}$. These empirical rates indicate a parametric error rate of $\sqrt{\frac{\dimOne^2 \log \dimOne}{n}}$ for $\|\hThetan - \ThetaStar\|_{\mathrm{F}}$ which matches the theoretical rate from  Corollary \ref{coro_main} (when $\dimTwo$ is treated as a constant).
\section{Conclusion and Remarks}
\label{sec:misc}

\noindent{\bf Conclusion.}
In this work, we provide a computationally efficient alternative to maximum likelihood estimation to learn distributions in a \textit{minimal truncated} $k$-parameter exponential family from i.i.d. samples. Our estimator is a minimizer of a novel convex loss function and can be viewed as an instance of   Bregman score as well as the {surrogate} likelihood. We provide rigorous finite sample analysis to achieve an $\alpha$-approximation to the true natural parameters with $O(\mathrm{poly}(k)/\alpha^2)$ samples. While our estimator is consistent and asymptotically normal, it is not asymptotically efficient. Investigating the possibility of a single estimator that achieves computational and asymptotic efficiency for this class of exponential family could be an interesting future direction.\\

\noindent{\bf Node-wise-sparse exponential family MRFs vs general exponential family.}  We highlight that the focus of our work is beyond the exponential families associated with node-wise-sparse MRFs, i.e., undirected graphical models, and towards general exponential families. The former focuses on  local assumptions on the parameters such as node-wise-sparsity, and the sample complexity depends logarithmically on the parameter dimension i.e., $O(\mathrm{log}(k))$.  In contrast, our work can handle local as well global structures on the parameters, e.g., a maximum norm constraint, a Frobenius norm constraint, or a nuclear norm constraint (see Section \ref{sec:experiments}), and our loss function in \eqref{eq:sampleGISMe} is a generalization of the interaction screening objective \citep{VuffrayMLC2016} and generalized interaction screening objective \citep{VuffrayML2019, ShahSW2021}. 
Similarly, for node-wise-sparse MRFs there has been a lot of work to relax the assumptions required for learning (see the discussion on Assumption \ref{lambdamin} below). Since our work focuses on global structures associated with the parameters, we leave the question of relaxing the assumptions required for learning as an open question.\\

	\noindent{\bf Assumption \ref{lambdamin}.} For node-wise-sparse pairwise exponential family MRFs (e.g., Ising models), which is a special case of the setting considered in our work, Assumption \ref{lambdamin} is proven (e.g., Appendix T.1 of \cite{ShahSW2021} provides one such analysis for a condition that is equivalent to Assumption \ref{lambdamin} for sparse continuous graphical model). However, such analysis typically requires a bound on the $\ell_1$ norm of the parameters associated with each node as in MRFs. Since the focus of our work is beyond the exponential families associated with node-wise-sparse MRFs, we view Assumption \ref{lambdamin} as an adequate condition to rule out certain singular distributions (as evident in the proof of Proposition \ref{prop:rsc_GISMe} where this condition is used to effectively lower bounds the variance of a non-constant random variable). Therefore, we expect this assumption to hold for most real-world applications. Further, we highlight that the MLE in \eqref{eq:mle} remains computationally intractable even under Assumption \ref{lambdamin}. To see this, one could again focus on node-wise-sparse pairwise exponential family MRFs where Assumption \ref{lambdamin} is proven and the MLE is still known to be computationally intractable.\\

\noindent{\bf  Beyond Truncated Exponential Family.}
While truncated exponential family is an important class of distributions, it requires boundedness of the support and does not capture a few widely used non-compact distributions i.e., distributions with infinite support (e.g., Gaussian distribution, Laplace distribution). While, conceptually, most non-compact distributions could be truncated by introducing a controlled amount of error, we believe this assumption could be lifted as for exponential families: $\mathbb{P}(|\rvx_i| \geq \delta  \log \gamma) \leq c\gamma^{-\delta}$ where $c>0$ is a constant and $\gamma > 0$. Alternatively, the notion of multiplicative regularizing distribution from \cite{RenMVL2021} could also be used. We believe extending our work to the non-compact setup could be an exciting direction for future work.

\section*{Acknowledgements}
	This work was supported, in part, by NSR under Grant No.\ CCF-1816209, ONR under Grant No. N00014-19-1-2665, the NSF TRIPODS Phase II grant towards Foundations of Data Science Institute, the MIT-IBM project on time series anomaly detection, and the KACST project on Towards Foundations of Reinforcement Learning.

\bibliographystyle{abbrvnat}
{\small
	\bibliography{Mybib_papers}
	
}
\clearpage
\appendix
\section*{Appendix}
\section{Related Works}
\label{appendix:related_works}
In this Appendix, we review additional works on exponential family Markov random fields, score-based methods and non-parametric density estimation, 
as well as review the related literature on Stein discrepancy.
\subsection{Exponential Family Markov Random Fields}
\label{subsubsec_exp_fam_mrf_rl_app}
Having reviewed some of the works on sparse exponential family MRFs in Section \ref{subsec_related_work}, we present here a brief overview of a few other works on the same.

Following the lines of \cite{YangRAL2015},  \cite{SuggalaKR2017} proposed an $\ell_1$ regularized node-conditional log-likelihood to learn the node-conditional density in \eqref{eq:conditional} for non-linear $\phi(\cdot)$. They used an alternating minimization technique and proximal gradient descent to solve the resulting optimization problem. However, their analysis required restricted strong convexity, bounded domain of the variables, non-negative node parameters, and hard-to-verify assumptions on gradient of the population loss. 

\cite{YangNL2018} introduced a non-parametric component to the node-conditional density in \eqref{eq:conditional}  while focusing on linear $\phi(\cdot)$. More specifically, they focused on the following joint density: 
\begin{align}
f_{\rvbx}(\svbx) \propto \exp \Big( \sum_{i \in [p]}  \eta_i(x_i) + \sum_{j \neq i} \theta_{ij}  x_i x_j \Big),  \label{eq:yang}
\end{align}
where $\eta_i(\cdot)$ is the non-parametric node-wise term. They proposed a node-conditional pseudo-likelihood (introduced in \cite{NingZL2017}) regularized by a nonconvex penalty and an adaptive multi-stage convex relaxation method to solve the resulting optimization problem. However, their finite-sample bounds require bounded moments of the variables, sparse eigenvalue condition on their loss function, and local smoothness of the log-partition function.	
\cite{SunKX2015} investigated infinite dimensional sparse pairwise exponential family MRFs where they assumed that the node and edge potentials lie in a Reproducing Kernel Hilbert space  (RKHS). They used a penalized version of the score matching objective of \cite{HyvarinenD2005}. However, their finite-sample analysis required incoherence and dependency conditions  \citep{WainwrightRL2006, JalaliRVS2011}. \cite{LinDS2016} considered the joint distribution in \eqref{eq:yuan} restricting the variables to be non-negative. They proposed a group lasso regularized generalized score matching objective \citep{Hyvarinen2007} which is a generalization of the score matching objective \citep{HyvarinenD2005} to non-negative data. However, their finite-sample analysis required the incoherence condition.

\subsection{Score-based and Stein discrepancy methods}
\label{subsubsec_score_based_methods_rl_app}
Having mentioned the principle behind and an example for the score-based method in Section \ref{subsec_related_work}, we briefly review a few other score-based methods in relation to the Stein discrepancy.

Stein discrepancy is a quantitative measure of how well a predictive density $q(\cdot)$ fits the density of interest $p(\cdot)$ based on the classical Stein's identity. Stein's identity defines an infinite number of identities indexed by a critic function $f$ and does not require evaluation of the partition function like the {score matching} method. By focusing on Stein discrepancy constructed from a RKHS, \cite{LiuLJ2016} and \cite{ChwialkowskiSG2016} independently proposed the kernel Stein discrepancy as a test statistic to access the goodness-of-fit for unnormalized densities.  \cite{LiuLJ2016}  and \cite{BarpBDGM2019} showed that the Fisher divergence, which was the minimization criterion used by the {score matching} method, can be viewed a special case of the kernel Stein discrepancy with a specific, fixed critic function $f$. \cite{BarpBDGM2019} showed that a few other methods (including the contrastive divergence by \cite{Hinton2002}) can also be viewed as a kernel Stein discrepancy with respect to a different class of critics. Despite the kernel Stein discrepancy being a natural criterion for fitting computationally hard models, there is no clear objective for choosing the right kernel and the kernels typically chosen (e.g. \cite{SunKX2015, StrathmannSLSG2015, SriperumbudurFGHK2017, SutherlandSAG2018}) are insufficient for complex datasets as pointed out by \cite{WenliangSSG2019}.

\cite{DaiLDHGSS2019} exploited the primal-dual view of the MLE to avoid estimating the normalizing constant at the price of introducing dual variables to be jointly estimated. They showed that many other methods including the contrastive divergence by \cite{Hinton2002}, pseudo-likelihood by \cite{Besag1975}, score matching by \cite{HyvarinenD2005} and minimum Stein discrepancy estimator by \cite{LiuLJ2016}, \cite{ChwialkowskiSG2016}, and \cite{BarpBDGM2019} are special cases of their estimator. However, this method results in expensive optimisation problems since they rely on adversarial optimisation (see \cite{RhodesXG2020} for details). {\cite{LiuKJC2019} proposed an inference method for unnormalized models known as discriminative likelihood estimator. This estimator follows the KL divergence minimization criterion and is implemented via density ratio estimation and a Stein operator. However, this method requires certain hard-to-verify conditions.}

\subsection{Non-parametric density estimation}
\label{subsubsec_non_para_rl_app}
Having reviewed some of the works on non-parametric density estimation in Section \ref{subsec_related_work}, we present here a brief overview of a few other works on the same.

\cite{Silverman1982} proposed to estimate the log density $\eta(\cdot) = \log f_{\rvbx}(\cdot)$ which is free of the positivity constraint and to augment the formulation of  \cite{Leonard1978} in \eqref{eq:leonard} by a functional $\int_{\svbx \in \cX} \exp(\eta(\svbx)) d\svbx$ to effectively enforce the unity constraint. They studied the theoretical properties of the penalized log likelihood estimator 
\begin{align}
\arg \min - \frac{1}{n} \sum_{i = 1}^{n} \eta(\svbx_i) +  \int_{\svbx \in \cX} \exp(\eta(\svbx)) d\svbx + \lambda J(\eta), \label{eq:silverman}
\end{align}
in the setting where $\eta(\cdot)$ lies in a Reproducing Kernel Hilbert space  (RKHS). The formulation in \eqref{eq:silverman} was further analyzed by \cite{O1988} who provided a pratical algorithm with cross-validated $\lambda$. However, similar to the formulation of  \cite{Leonard1978}, the formulation of  \cite{Silverman1982} scales poorly in high dimensions.

The formulation of \cite{Leonard1978} in \eqref{eq:leonard} also evolved through a series of works as described below.  \cite{GuQ1993} studied the theoretical properties of the estimator in \eqref{eq:leonard} over a Reproducing Kernel Hilbert Space (RKHS). However, they considered a finite-dimensional function space approximation (consisting of the linear span of kernel function) to the RKHS. \cite{Gu1993} provided a pratical algorithm with cross-validated $\lambda$ for the estimator analyzed by \cite{GuQ1993}. \cite{GuW2003} improved upon the algorithm in \cite{Gu1993} by providing a direct strategy for cross-validation. However, this function space approximation does not enjoy strong statistical guarantees.

\cite{GuJL2013} provided a practical way for choosing a cross-validated $\lambda$ for the formulation of  \cite{JeonL2006} in \eqref{eq:jeon} using the function space approximation of \cite{GuQ1993}. However, their algorithm suffers in high dimensions where performing accurate estimation is challenging.

\section{Smoothness of the loss function}
\label{appendix:algorithm_section_proofs}%
In this Appendix, we state and prove the smoothness property of the loss function $\cL_{n}(\cdot)$. 

\begin{proposition}\label{proposition:smoothness}
	Consider any $\Theta \in \Reals^{\dimOne \times  \dimTwo}$ such that $R(\Theta) \leq \bar{r}$. Under Assumptions \ref{bounds_parameter}, \ref{bounds_statistics} and \ref{bounds_statistics_maximum}, $\cL_{n}(\Theta)$ is a $4 \dimOne \dimTwo \phiMax^2\exp(2\bar{r}\bd )$ smooth function of $\Theta$.
\end{proposition}
\begin{proof}[Proof of Proposition \ref{proposition:smoothness}]
	To show $4 \dimOne \dimTwo \phiMax^2\exp(2\bar{r}\bd )$ smoothness of $\cL_{n}(\Theta)$,
	we  show that the largest eigenvalue of the Hessian\footnote[5]{Ideally, one would consider the Hessian of $\cL_{n}(\vect({\Theta}))$. However, we abuse the terminology for the ease of the exposition.} of $\cL_{n}(\Theta)$ is upper bounded by $4 \dimOne \dimTwo \phiMax^2\exp(2\bar{r}\bd )$. First, we simplify the Hessian of $\cL_{n}(\Theta)$, i.e., $\nabla^2 \cL_{n}(\Theta)$. 
	The component of the Hessian of $\cL_{n}(\Theta)$ corresponding to $\Theta_{u_1v_1}$
	and $\Theta_{u_2v_2}$ for some $u_1,u_2 \in [\dimOne]$ and $v_1,v_2 \in [ \dimTwo]$ is given by
	\begin{align}
	\frac{\partial^2 \cL_{n}(\Theta)}{\partial \Theta_{u_1v_1} \partial \Theta_{u_2v_2}}   = \frac{1}{n} \sum_{t = 1}^{n} \varPhi_{u_1v_1}(\svbx^{(t)}) \varPhi_{u_2v_2}(\svbx^{(t)}) \exp\Big( \!\!-\!\big\langle \big\langle  \Theta, \varPhi(\svbx^{(t)}) \big\rangle \big\rangle \Big).  \label{eq:hessian-GISMe}
	\end{align}
	From the Gershgorin circle theorem, we know that the largest eigenvalue of any matrix is upper bounded by the largest absolute row sum or column sum. Let $\lambda_{\max}(\nabla^2 \cL_{n}(\Theta))$ denote the largest eigenvalue of $\nabla^2 \cL_{n}(\Theta)$. We have the following
	\begin{align}
	\lambda_{\max}(\nabla^2 \cL_{n}(\Theta)) & \leq \max_{u_2,v_2} \sum_{u_1,v_1} \bigg| \frac{\partial^2 \cL_{n}(\Theta)}{\partial \Theta_{u_1v_1} \partial \Theta_{u_2v_2}}  \bigg| \\
	& = \max_{u_2,v_2} \sum_{u_1,v_1} \bigg|  \frac{1}{n} \sum_{t = 1}^{n} \varPhi_{u_1v_1}(\svbx^{(t)}) \varPhi_{u_2v_2}(\svbx^{(t)}) \exp\Big(\!\! -\!\big\langle \big\langle  \Theta, \varPhi(\svbx^{(t)}) \big\rangle \big\rangle \Big) \bigg|. \label{eq_hessian_bound}
	\end{align}
	To bound \eqref{eq_hessian_bound},  we first bound the absolute inner product between $\Theta$ and $\varPhi$, i.e., $\big|\big\langle \big\langle \Theta, \varPhi(\svbx) \big\rangle \big\rangle\big|$ for any $\svbx \in \cX$. We have
	\begin{align}
	\big|\big\langle \big\langle \Theta, \varPhi(\svbx) \big\rangle \big\rangle\big| \stackrel{(a)}{\leq}   \cR (\Theta) \!\times\! \cR^*(\varPhi(\svbx)) & \stackrel{(b)}{\leq} \bar{r} \Big( \cR^*(\Phi(\svbx)) \!+\! \cR^*(\Expectation_{\Uniform} [\Phi(\rvbx)])  \Big)  \\
	& \stackrel{(c)}{\leq} \bar{r} \Big( \cR^*(\Phi(\svbx)) \!+\! \Expectation_{\Uniform}[\cR^*(\Phi(\rvbx))] \Big) \stackrel{(d)}{\leq} 2\bar{r} \bd, \label{eq:bound_inner_product}
	\end{align}
	where $(a)$ follows from the definition of a dual norm, $(b)$ follows from \eqref{eq:CSS}, the triangle inequality and because $R(\Theta) \leq \bar{r}$, $(c)$ follows from convexity of norms, and $(d)$ follows from Assumption \ref{bounds_statistics}. Likewise, we can bound $\|\varPhi(\svbx)\|_{\max}$ by $2 \phiMax$ for any $\svbx \in \cX$ using Assumption \ref{bounds_statistics_maximum}. Using these bounds in \eqref{eq_hessian_bound}, we have
	\begin{align}
	\lambda_{\max}(\nabla^2 \cL_{n}(\Theta)) \leq  \max_{u_2,v_2} \sum_{u_1,v_1} 4 \phiMax^2  \exp(2\bar{r}\bd)  = 4 \dimOne \dimTwo \phiMax^2\exp(2\bar{r}\bd ).
	\end{align}
	Therefore, $\cL_{n}(\Theta)$ is a $4 \dimOne \dimTwo \phiMax^2\exp(2\bar{r}\bd )$ smooth function of $\Theta$.
\end{proof}
\section{Proof of Theorem 4.2}
\label{appendix:proof of thm:consistency_normality}%
In this Appendix, we prove Theorem \ref{thm:consistency_normality} by using the theory of $M$-estimation. In particular, observe that $\hThetan$ is an $M$-estimator i.e., $\hThetan$ is a sample average. Then, we invoke Theorem 4.1.1 and Theorem 4.1.3 of \cite{Amemiya1985} to prove the consistency and the asymptotic normality of $\hThetan$, respectively. We restate Theorem \ref{thm:consistency_normality} below and then provide the proof. Recall that  $A(\ThetaStar)$ denotes the covariance matrix of $\vect\big(\varPhi(\rvbx)\exp\big( -\big\langle \big\langle \ThetaStar, \varPhi(\rvbx) \big\rangle \big\rangle \big)\big)$ and $B(\ThetaStar)$ denotes the cross-covariance matrix of $\vect(\varPhi(\rvbx))$ and $\vect(\varPhi(\rvbx) \exp\big( -\big\langle \big\langle \ThetaStar, \varPhi(\rvbx) \big\rangle \big\rangle \big))$.

\theoremconsistencynormality*
\begin{proof}[Proof of Theorem \ref{thm:consistency_normality}] We divide the proof in two parts.\\
	
	{\bf Consistency. } We first show that $\hThetan$ is asymptotically consistent. To show this, recall Theorem 4.1.1 of \cite{Amemiya1985}.
	
	\citet[Theorem~4.1.1]{Amemiya1985}: Let $z_1, \cdots, z_n$ be i.i.d. samples of a random variable $\rvz$. Let $q(\rvz ; \theta)$ be some function of $\rvz$ parameterized by $\theta \in \Upsilon$. Let $\theta^*$ be the true underlying parameter. Define
	\begin{align}
	Q_n(\theta) = \frac{1}{n} \sum_{i = 1}^{n} q(z_i ; \theta) \qquad \text{and} \qquad
	\hthetan \in \argmin_{\theta \in \Upsilon} Q_n(\theta). \label{eq:m-est-con}
	\end{align} 
	Let the following be true.
	\begin{enumerate}[leftmargin=6mm, itemsep=-0.5mm]
		\item[(a)] $\Upsilon$ is compact,
		\item[(b)] $Q_n(\theta)$ converges uniformly in probability to a non-stochastic function $Q(\theta)$, 
		\item[(c)] $Q(\theta)$ is continuous, and
		\item[(d)] $Q(\theta)$ is uniquely minimzed at $\theta^*$.
	\end{enumerate}
Then, $\hthetan$ is consistent for $\theta^*$ i.e., $\hthetan \stackrel{p}{\to} \theta^*$ as $n\to \infty$.\\

Letting $z \coloneqq \rvbx$, $\theta \coloneqq \Theta$, $\hthetan \coloneqq \hThetan$, $\theta^* \coloneqq \ThetaStar$, $\Upsilon = \parameterSet$, $q(z; \theta) \coloneqq \exp\big( -\big\langle \big\langle \Theta, \varPhi(\svbx) \big\rangle \big\rangle \big)$, and $Q_n(\theta) \coloneqq \cL_{n}(\Theta)$, it is sufficient to show the following:
\begin{enumerate}[leftmargin=6mm, itemsep=-0.5mm]
	\item[(a)] $\parameterSet$ is compact,
	\item[(b)] $\cL_{n}(\Theta)$ converges uniformly in probability to a non-stochastic function $\cL(\Theta)$, 
	\item[(c)] $\cL(\Theta)$ is continuous, and
	\item[(d)] $\cL(\Theta)$ is uniquely minimzed at $\ThetaStar$.
\end{enumerate}

Let us show these one by one.
\begin{enumerate}[leftmargin=6mm, itemsep=-0.5mm]
	\item[(a)] We have $\parameterSet = \{\Theta : \cR(\Theta) \leq r\}$ which is bounded and closed. Therefore, $\parameterSet$ is compact.
	
	\item[(b)] Recall \citet[Theorem 2]{Jennrich1969}: Let $z_1, \cdots, z_n$ be i.i.d. samples of a random variable $\rvz$. Let $g(\rvz ; \theta)$ be a  function of $\theta$ parameterized by $\theta \in \Upsilon$.  Then, $n^{-1} \sum_t g(z_t , \theta)$ converges uniformly in probability to $\Expectation [ g(\rvz, \theta)]$ if 
	\begin{enumerate}[leftmargin=6mm]
		\item[(i)] $\Upsilon$ is compact, 
		\item[(ii)] $g(\rvz , \theta)$ is continuous at each $\theta \in \Upsilon$ with probability one, 
		\item[(iii)] $g(\rvz , \theta)$ is dominated by a function $G(\rvz)$ i.e., $| g(\rvz , \theta) | \leq G(\rvz)$, and
		\item[(iv)] $\Expectation[G(\rvz)] < \infty$.
	\end{enumerate}

	Using this theorem with $\rvz \coloneqq \rvbx$, $\theta \coloneqq \Theta$, $\Upsilon \coloneqq \parameterSet$, $g(\rvz, \theta) \coloneqq \exp\big( -\big\langle \big\langle \Theta, \varPhi(\svbx) \big\rangle \big\rangle \big)$, $G(\rvz) \coloneqq \exp(2\br \bd)$ and \eqref{eq:bound_inner_product}, we conclude that $\cL_{n}(\Theta)$ converges to $\cL(\Theta)$ uniformly in probability.
	
	\item[(c)] $\exp\big( -\big\langle \big\langle \Theta, \varPhi(\svbx) \big\rangle \big\rangle \big)$ is a continuous function of $\Theta \in \parameterSet$. Further, $\DensityXTrue$ does not functionally depend on $\Theta$. Therefore, we have continuity of $\cL(\Theta)$ for all $\Theta \in \parameterSet$.
	
	\item[(d)] From Theorem \ref{theorem:GRISMe-KLD}, $\cL(\Theta)$ is uniquely minimized at $\ThetaStar$.
\end{enumerate}
		
Therefore, we have asymptotic consistency of $\hThetan$.\\

	{\bf Normality. }
	We now show that $\hThetan$ is asymptotically normal. To show this, recall Theorem 4.1.3 of \cite{Amemiya1985}.
	
	\citet[Theorem~4.1.3]{Amemiya1985}: Let $z_1, \cdots, z_n$ be i.i.d. samples of a random variable $\rvz$. Let $q(\rvz ; \theta)$ be some function of $\rvz$ parameterized by $\theta \in \Upsilon$. Let $\theta^*$ be the true underlying parameter. Define
	\begin{align}
	Q_n(\theta) = \frac{1}{n} \sum_{i = 1}^{n} q(z_i ; \theta) \qquad \text{and} \qquad
	\hthetan \in \argmin_{\theta \in \Upsilon} Q_n(\theta). \label{eq:m-est-eff}
	\end{align} 
	Let the following be true.
	\begin{enumerate}[leftmargin=6mm, itemsep=-0.5mm]
		\item[(a)] $\hthetan$ is consistent for $\theta^*$, 
		\item[(b)] $\theta^*$ lies in the interior of the parameter space $\Upsilon$, 
		\item[(c)] $Q_n$ is twice continuously differentiable in an open and convex neighbourhood of $\theta^*$, 
		\item[(d)] $\sqrt{n}\nabla Q_n(\theta)|_{\theta = \theta^*} \stackrel{d}{\to} {\cal N}({\bf 0}, A(\theta^*))$, and 
		\item[(e)] $\nabla^2 Q_n(\theta)|_{\theta = \hthetan} \stackrel{p}{\to} B(\theta^*)$ with $B(\theta)$ finite, non-singular, and continuous at $\theta^*$, 
	\end{enumerate}
Then, $\hthetan$ is normal for $\theta^*$ i.e., $\sqrt{n}( \hthetan - \theta^*)\stackrel{d}{\to} {\cal N}({\bf 0}, B^{-1}(\theta^*)A(\theta^*)B^{-1}(\theta^*))$.\\

Letting $z \coloneqq \rvbx$, $\theta \coloneqq \Theta$, $\hthetan \coloneqq \hThetan$, $\theta^* \coloneqq \ThetaStar$, $\Upsilon = \parameterSet$, $q(z; \theta) \coloneqq \exp\big( -\big\langle \big\langle \Theta, \varPhi(\svbx) \big\rangle \big\rangle \big)$, and $Q_n(\theta) \coloneqq \cL_{n}(\Theta)$, it is sufficient to show the following:
\begin{enumerate}[leftmargin=6mm, itemsep=-0.5mm]
	\item[(a)] $\hThetan$ is consistent for $\ThetaStar$, 
	\item[(b)] $\ThetaStar$ lies in the interior of the parameter space $\parameterSet$, 
	\item[(c)] $\cL_{n}$ is twice continuously differentiable in an open and convex neighbourhood of $\ThetaStar$, 
	\item[(d)] $\sqrt{n}\nabla \cL_{n}(\vect(\Theta))|_{\Theta = \ThetaStar} \stackrel{d}{\to} {\cal N}({\bf 0}, A(\ThetaStar))$, and 
	\item[(e)] $\nabla^2 \cL_{n}(\vect(\Theta))|_{\Theta = \hThetan} \stackrel{p}{\to} B(\ThetaStar)$ with $B(\Theta)$ finite, non-singular, and continuous at $\ThetaStar$, 
\end{enumerate}

Let us show these one by one.

\begin{enumerate}[leftmargin=6mm, itemsep=-0.5mm]
			\item[(a)] We have established that $\hThetan$ is consistent for $\ThetaStar$ in the first half of the proof.
			\item[(b)] The assumption that $\ThetaStar \in \text{interior}(\parameterSet)$ is equivalent to $\ThetaStar$ belonging to the interior of $\parameterSet$.
			\item[(c)] Fix $u_1,u_2 \in [\dimOne]$ and $v_1,v_2 \in [ \dimTwo]$. We have
			\begin{align}
			\frac{\partial^2 \cL_{n}(\Theta)}{\partial \Theta_{u_1v_1} \partial\Theta_{u_2v_2}}  = \frac{1}{n} \sum_{t = 1}^{n} \varPhi_{u_1v_1}(\svbx^{(t)}) \varPhi_{u_2v_2}(\svbx^{(t)}) \exp\big( -\big\langle \big\langle \Theta, \varPhi(\svbx^{(t)}) \big\rangle \big\rangle \big). 
			\end{align}
		Thus, $\partial^2 \cL_{n}(\Theta)/\partial \Theta_{u_1v_1} \partial \Theta_{u_2v_2}$ exists. Using the continuity of $\varPhi(\cdot)$ and 
		$\exp\big( -\big\langle \big\langle \Theta, \varPhi(\cdot) \big\rangle \big\rangle \big)$, we see that $\partial^2 \cL_{n}(\Theta)/\partial \Theta_{u_1v_1} \partial \Theta_{u_2v_2}$ is continuous in an open and convex neighborhood of $\ThetaStar$.
		\item[(d)] For any $u \in [\dimOne]$ and $v \in [ \dimTwo]$, define the random variable
		\begin{align}
		\rvx_{uv}  = - \varPhi_{uv}(\svbx) \exp\big( -\big\langle \big\langle \ThetaStar, \varPhi(\svbx) \big\rangle \big\rangle \big).
		\end{align}
		The component of the gradient of $\cL_{n}(\vect(\Theta))$ corresponding to $\Theta_{uv}$
		evaluated at $\ThetaStar$ is given by
		\begin{align}
		\frac{\partial \cL_{n}(\ThetaStar)}{\partial \Theta_{uv}}  = - \frac{1}{n} \sum_{t = 1}^{n} \varPhi_{uv}(\svbx^{(t)}) \exp\big( -\big\langle \big\langle \ThetaStar, \varPhi(\svbx^{(t)}) \big\rangle \big\rangle \big). 
		\end{align}
				Each term in the above summation is distributed as per the random variable $\rvx_{uv} $.
				The random variable $\rvx_{uv}$ has zero mean (see Lemma \ref{lemma:zero_expectation}). Using this and the multivariate central limit theorem \citep{Vaart2000}, we have
				\begin{align}
				\sqrt{n} \nabla \cL_{n}(\vect(\Theta)) |_{\Theta = \ThetaStar} \xrightarrow{d} {\cal N}({\bf 0}, A(\ThetaStar)),
				\end{align} 
				where $A(\ThetaStar)$ is the covariance matrix of $\vect\big(\varPhi(\rvbx)\exp\big( -\big\langle \big\langle \ThetaStar, \varPhi(\rvbx) \big\rangle \big\rangle \big)\big).$
	\item [(e)] We start by showing that the following is true.
	\begin{align}
	\nabla^2 \cL_{n}(\vect(\Theta)) |_{\Theta = \hThetan} \xrightarrow{p} \nabla^2 \cL(\vect(\Theta)) |_{\Theta = \ThetaStar}.  \label{eq:ulln+cmt}
	\end{align}
	To begin with, using the uniform law of large numbers \cite[Theorem 2]{Jennrich1969} for any $\Theta \in \parameterSet$ results in
	\begin{align}
	\nabla^2 \cL_{n}(\vect(\Theta))  \xrightarrow{p} \nabla^2 \cL(\vect(\Theta)). \label{eq:ulln} 
	\end{align}
	Using the consistency of $\hThetan$ and the continuous mapping theorem, we have
	\begin{align}
	\nabla^2 \cL(\vect(\Theta)) |_{\Theta = \hThetan} \xrightarrow{p} \nabla^2 \cL(\vect(\Theta)) |_{\Theta = \ThetaStar}.  \label{eq:cmt}
	\end{align}
	Let $u_1,u_2 \in [\dimOne]$ and $v_1,v_2 \in [ \dimTwo]$. From \eqref{eq:ulln} and \eqref{eq:cmt}, for any $\epsilon > 0$, for any $\delta > 0$, there exists integers $n_1 , n_2$ such that for $n \geq \max\{n_1,n_2\}$ we have,
	\begin{align}
	\Prob( | \partial^2 \cL_{n}(\hThetan)/\partial \Theta_{u_1v_1} \partial \Theta_{u_2v_2}  - \partial^2 \cL(\hThetan)/\partial \Theta_{u_1v_1} \partial \Theta_{u_2v_2}  | > \epsilon / 2 ) \leq \delta / 2
	\end{align}
	and 
	\begin{align}
	\Prob( | \partial^2 \cL(\hThetan)/\partial \Theta_{u_1v_1} \partial \Theta_{u_2v_2}  - \partial^2 \cL(\ThetaStar)/\partial \Theta_{u_1v_1} \partial \Theta_{u_2v_2}  | > \epsilon / 2 ) \leq \delta / 2.
	\end{align} 
	Now for $n \geq \max\{n_1,n_2\}$, using the triangle inequality we have
	\begin{align}
	\Prob( | \partial^2 \cL_{n}(\hThetan)/\partial \Theta_{u_1v_1} \partial \Theta_{u_2v_2}  - \partial^2 \cL(\ThetaStar)/\partial \Theta_{u_1v_1} \partial \Theta_{u_2v_2}  | > \epsilon)  \leq \delta / 2 + \delta / 2 = \delta.
	\end{align}
	Thus, we have \eqref{eq:ulln+cmt}. Using the definition of $\cL(\Theta)$, we have
	\begin{align}
	\partial^2 \cL(\ThetaStar)/\partial \Theta_{u_1v_1} \partial \Theta_{u_2v_2} & = \Expectation \Big[ \varPhi_{u_1v_1}(\rvbx) \varPhi_{u_2v_2}(\rvbx)\exp\big( -\big\langle \big\langle \ThetaStar, \varPhi(\rvbx) \big\rangle \big\rangle \big)  \Big] \\
	& \stackrel{(a)}{=} \Expectation \Big[ \varPhi_{u_1v_1}(\rvbx) \varPhi_{u_2v_2}(\rvbx)\exp\big( -\big\langle \big\langle \ThetaStar, \varPhi(\rvbx) \big\rangle \big\rangle \big) \Big]  \\ 
	& \qquad - \Expectation \Big[ \varPhi_{u_1v_1}(\rvbx)\Big]  \Expectation \Big[ \varPhi_{u_2v_2}(\rvbx)\exp\big( -\big\langle \big\langle \ThetaStar, \varPhi(\rvbx) \big\rangle \big\rangle\big)\Big]\\
	& = \text{cov} \Big(  \varPhi_{u_1v_1}(\rvbx) , \varPhi_{u_2v_2}(\rvbx)\exp\big( -\big\langle \big\langle \ThetaStar, \varPhi(\rvbx) \big\rangle \big\rangle \big)\Big),
	\end{align}	
	where (a) follows because $\Expectation \big[ \varPhi_{u_2v_2}(\rvbx)\exp\big( -\big\langle \big\langle \ThetaStar, \varPhi(\rvbx) \big\rangle \big\rangle\big)\big] = 0$ for any $u_2 \in [\dimOne]$ and $v_2 \in [ \dimTwo]$ from Lemma \ref{lemma:zero_expectation}.	Therefore, we have
	\begin{align}
	\nabla^2 \cL_{n}(\vect(\Theta)) |_{\Theta = \hThetan} \xrightarrow{p} B(\ThetaStar),
	\end{align}
	where $B(\ThetaStar)$ is the cross-covariance matrix of $\vect(\varPhi(\rvbx))$ and $\vect\big(\varPhi(\rvbx) \exp\big( -\big\langle \big\langle \ThetaStar, \varPhi(\rvbx) \big\rangle \big\rangle \big)\big)$. Finiteness and continuity of $\varPhi(\rvbx)$ and $\varPhi(\rvbx)\exp\big( -\big\langle \big\langle \ThetaStar, \varPhi(\rvbx) \big\rangle \big\rangle \big)$ implies the finiteness and continuity of $B(\ThetaStar)$. By assumption,  the cross-covariance matrix of $\vect(\varPhi(\rvbx))$ and $\vect\big(\varPhi(\rvbx) \exp\big( -\big\langle \big\langle \ThetaStar, \varPhi(\rvbx) \big\rangle \big\rangle \big)\big)$ is invertible.
\end{enumerate}	
Therefore, we have the asymptotic normality of $\hThetan$.
\end{proof}

\section{Restricted strong convexity of the loss function} 
\label{appendix:rsc}
In this Appendix, we show that, with enough samples, the loss function obeys the restricted strong convexity property with high probability. This result in turn allows us to prove Theorem \ref{thm:finite_sample} in Appendix \ref{appendix_proof_finite_sample}. We first state the main result of this Appendix (Proposition \ref{prop:rsc_GISMe}). Next, we introduce the notion of correlation for the centered natural statistics and provide a supporting Lemma wherein we bound the deviation between the true correlation and the empirical correlation. Finally, we prove Proposition \ref{prop:rsc_GISMe}.

For any $\Theta \in \Reals^{\dimOne \times  \dimTwo}$ with $\Delta = \Theta - \ThetaStar$, define the residual of the first-order Taylor expansion as 
\begin{align}
\delta \cL_{n}(\Delta, \ThetaStar) = \cL_{n}(\ThetaStar + \Delta) - \cL_{n}(\ThetaStar)  - \langle \langle \nabla \cL_{n}(\ThetaStar),\Delta \rangle\rangle.  \label{eq:residual}
\end{align}
\begin{restatable}{proposition}{proprsc}\label{prop:rsc_GISMe}
	Let Assumptions \ref{bounds_parameter}, \ref{bounds_statistics}, \ref{bounds_statistics_maximum} and \ref{lambdamin} be satisfied. Consider any $\Theta \in \Reals^{\dimOne \times  \dimTwo}$ such that 
	$\Delta \in 4 \parameterSet$ and
	$\|\Delta\|_{1,1} \leq \gamma(\dimOne, \dimTwo) \|\Delta\|_\mathrm{{F}}$ where $\Delta = \Theta - \ThetaStar$ and $\gamma(\cdot, \cdot)$ is some function of $\dimOne$ and $\dimTwo$. For any $\deltaThree \in (0,1)$,  the residual defined in \eqref{eq:residual} satisfies
	\begin{align}
	\delta \cL_{n}(\Delta, \ThetaStar) \geq \frac{\lambdaMin \exp(-2\br\bd )}{4 + 16\br\bd} \|\Delta\|^2_{\mathrm{F}}, 
	\end{align}
	with probability at least $1-\deltaThree$ as long as 
	\begin{align}
	n > \frac{128 \phiMax^4 \gamma^4(\dimOne, \dimTwo) }{\lambdaMin^2}\log\Big(\frac{2 \dimOne^2  \dimTwo^2 }{\deltaThree}\Big). 
	\end{align}
\end{restatable}

\subsection{Correlation between centered natural statistics}
\label{subsec:correlation between centered natural statistics}
For any $u_1,u_2 \in [\dimOne]$ and $v_1,v_2 \in [ \dimTwo]$, let $H_{u_1v_1u_2v_2}$ denote the correlation between $\varPhi_{u_1v_1}(\rvbx)$ and $\varPhi_{u_2v_2}(\rvbx)$ defined as
\begin{align}
H_{u_1v_1u_2v_2} = \Expectation \big[\varPhi_{u_1v_1}(\rvbx)\varPhi_{u_2v_2}(\rvbx)\big],  \label{eq:population_correlation}
\end{align}
and let $\bH = [H_{u_1v_1u_2v_2}] \in \Reals^{[\dimOne] \times [ \dimTwo]  \times [\dimOne] \times [ \dimTwo] }$ be the corresponding correlation tensor. 
Similarly, we define $\hbH$  based on the empirical estimates of the correlation
\begin{align}
\hH_{u_1v_1u_2v_2} = \frac{1}{n} \sum_{t=1}^{n} \varPhi_{u_1v_1}(\svbx^{(t)})\varPhi_{u_2v_2}(\svbx^{(t)}). \label{eq:empirical_correlation}
\end{align}

The following lemma bounds the deviation between the true correlation and the empirical correlation.
\begin{lemma} \label{lemma:correlation_concentration}
	Consider any $u_1,u_2 \in [\dimOne]$ and $v_1,v_2 \in [ \dimTwo]$. Let Assumption \ref{bounds_statistics_maximum} be satisfied. Then, we have for any $\epsTwo > 0$,
	\begin{align}
	|\hH_{u_1v_1u_2v_2} - H_{u_1v_1u_2v_2}| < \epsTwo,
	\end{align}
	with probability at least $1 - \deltaTwo$ as long as
	\begin{align}
	n > \frac{32 \phiMax^4}{\epsTwo^2}\log\Big(\frac{2 \dimOne^2  \dimTwo^2}{\deltaTwo}\Big).
	\end{align}
\end{lemma}
\begin{proof}[Proof of Lemma~\ref{lemma:correlation_concentration}]
	Fix $u_1,u_2 \in [\dimOne]$ and $v_1,v_2 \in [ \dimTwo]$. 
	The random variable defined as $Y_{u_1v_1u_2v_2} \coloneqq \varPhi_{u_1v_1}(\rvbx)\varPhi_{u_2v_2}(\rvbx)$ satisfies $|Y_{u_1v_1u_2v_2}| \leq 4\phiMax^2$ (from \eqref{eq:CSS}, triangle inequality, convexity of norms, and Assumption \ref{bounds_statistics_maximum}). 
	Using the Hoeffding inequality we get
	\begin{align}
	\Prob \left( |\hH_{u_1v_1u_2v_2} - H_{u_1v_1u_2v_2}| > \epsTwo \right) < 2\exp \left(-\frac{n \epsTwo^2}{32\phiMax^4}\right).
	\end{align}
	The proof follows by using the union bound over all $u_1,u_2 \in [\dimOne]$ and $v_1,v_2 \in [ \dimTwo]$.
\end{proof}

\subsection{Proof of Proposition~\ref{prop:rsc_GISMe}}

\begin{proof}[Proof of Proposition~\ref{prop:rsc_GISMe}]

	First, we simplify the gradient of $\cL_{n}(\Theta)$\footnote[6]{Ideally, one would consider the gradient of $\cL_{n}(\vect({\Theta}))$. However, for the ease of the exposition we abuse the terminology.} evaluated at $\ThetaStar$. 
	For any $u \in [\dimOne]$ and $v \in [ \dimTwo]$, the component of the gradient of $\cL_{n}(\Theta)$ corresponding to $\Theta_{uv}$
	evaluated at $\ThetaStar$ is given by
	\begin{align}
	\frac{\partial \cL_{n}(\ThetaStar)}{\partial \Theta_{uv}}  = -\frac{1}{n} \sum_{t = 1}^{n} \varPhi_{uv}(\svbx^{(t)}) \exp\big( -\big\langle\big\langle \ThetaStar, \varPhi(\svbx^{(t)}) \big\rangle \big\rangle \big).  \label{eq:gradient-GISMe}
	\end{align}
	
	We now provide the desired lower bound on the residual. Substituting \eqref{eq:sampleGISMe} and \eqref{eq:gradient-GISMe} in \eqref{eq:residual}, we have
	\begin{align}
	\delta \cL_{n}(\Delta, \ThetaStar) &= \frac{1}{n} \sum_{t = 1}^{n}  \exp\big( \hspace{-0.5mm}-\hspace{-0.5mm}\big\langle\big\langle \ThetaStar, \varPhi(\svbx^{(t)}) \big\rangle \big\rangle \big) \times \Big[ \exp\big( \hspace{-0.5mm}-\hspace{-0.5mm}\big\langle\big\langle \Delta, \varPhi(\svbx^{(t)}) \big\rangle \big\rangle \big) - \hspace{-0.5mm} 1 \hspace{-0.5mm} + \big\langle \big\langle \Delta, \varPhi(\svbx^{(t)}) \big\rangle \big\rangle \Big]  \\
	& \stackrel{(a)}{\geq} \exp(-2\br\bd) \times \frac{1}{n} \sum_{t = 1}^{n}  \Big[ \exp\big( -\big\langle\big\langle \Delta, \varPhi(\svbx^{(t)}) \big\rangle \big\rangle \big) - 1 +  \big\langle\big\langle \Delta, \varPhi(\svbx^{(t)}) \big\rangle \big\rangle \Big]  \\	
	& \stackrel{(b)}{\geq} \exp(-2\br\bd) \times \frac{1}{n} \sum_{t = 1}^{n}  \frac{\big|\big\langle\big\langle \Delta, \varPhi(\svbx^{(t)}) \big\rangle\big\rangle \big|^2}{2+\big|\big\langle\big\langle \Delta, \varPhi(\svbx^{(t)}) \big\rangle\big\rangle \big|}\\
	& \stackrel{(c)}{\geq} \frac{\exp(-2\br\bd)}{2+8\br\bd} \times \frac{1}{n} \sum_{t = 1}^{n}  \big|\big\langle\big\langle \Delta, \varPhi(\svbx^{(t)}) \big\rangle\big\rangle\big|^2\\
	& \stackrel{(d)}{=} \frac{\exp(-2\br\bd)}{2+8\br\bd} \times \sum_{u_1=1}^{\dimOne}\sum_{v_1=1}^{ \dimTwo}  \sum_{u_2=1}^{\dimOne}\sum_{v_2=1}^{ \dimTwo} 
	\Delta_{u_1v_1} \hH_{u_1v_1u_2v_2} \Delta_{u_2v_2}\\
	& = \frac{\exp(-2\br\bd)}{2+8\br\bd} \times \sum_{u_1=1}^{\dimOne}\sum_{v_1=1}^{ \dimTwo}  \sum_{u_2=1}^{\dimOne}\sum_{v_2=1}^{ \dimTwo} 
	 \Delta_{u_1v_1}  [H_{u_1v_1u_2v_2}  + \hH_{u_1v_1u_2v_2} - H_{u_1v_1u_2v_2}] \Delta_{u_2v_2},
	\end{align}
	where $(a)$ follows because $ - \big\langle\big\langle \ThetaStar, \varPhi(\svbx) \big\rangle \big\rangle \geq  -2\br\bd$ for every $\svbx \in \cX$ from \eqref{eq:bound_inner_product}, $(b)$ follows because $e^{-z} - 1 + z \geq \frac{z^2}{2 + |z|}$ for any $z \in \Reals$, $(c)$ follows because $ - \big\langle\big\langle \Delta, \varPhi(\svbx) \big\rangle \big\rangle \geq  -8\br\bd$ for every $\svbx \in \cX$ and $\Delta \in 4 \parameterSet$ from arguments similar to \eqref{eq:bound_inner_product}, and $(d)$ follows from \eqref{eq:empirical_correlation}.\\
	
	Let the number of samples satisfy
	\begin{align}
	n > \frac{128 \phiMax^4 \gamma^4(\dimOne, \dimTwo) }{\lambdaMin^2}\log\Big(\frac{2 \dimOne^2  \dimTwo^2 }{\deltaThree}\Big). 
	\end{align}
	Using Lemma \ref{lemma:correlation_concentration} with $\epsTwo \mapsfrom \frac{\lambdaMin}{2 \gamma^2(\dimOne, \dimTwo)}$ and $\deltaTwo \mapsfrom \deltaThree$, and the triangle inequality, we have the following with probability at least $1-\deltaThree$
	\begin{align}
	\delta \cL_{n}(\Delta, \ThetaStar) &\geq \frac{\exp(-2\br\bd)}{2+8\br\bd} \times \bigg[ \sum_{u_1=1}^{\dimOne}\sum_{v_1=1}^{ \dimTwo}  \sum_{u_2=1}^{\dimOne}\sum_{v_2=1}^{ \dimTwo} 
	\Delta_{u_1v_1} H_{u_1v_1u_2v_2} \Delta_{u_2v_2} 
	- \frac{\lambdaMin}{2 \gamma^2(\dimOne, \dimTwo)} \|\Delta\|^2_{1,1}  \bigg]\\
	&\stackrel{(a)}{\geq} \frac{\exp(-2\br\bd)}{2+8\br\bd} \times \bigg[ \sum_{u_1=1}^{\dimOne}\sum_{v_1=1}^{ \dimTwo}  \sum_{u_2=1}^{\dimOne}\sum_{v_2=1}^{ \dimTwo} 
	  \Delta_{u_1v_1} H_{u_1v_1u_2v_2} \Delta_{u_2v_2}  - \frac{\lambdaMin}{2} \|\Delta\|^2_{\mathrm{F}}  \bigg]\\
	& \stackrel{(b)}{=} \frac{\exp(-2\br\bd)}{2+8\br\bd} \times \Big[ \vect(\Delta) \Expectation[\vect(\varPhi(\rvbx))\vect(\varPhi(\rvbx))^T]\vect(\Delta)^T  - \frac{\lambdaMin}{2} \|\Delta\|^2_{\mathrm{F}}  \Big]\\
	& \stackrel{(c)}{\geq} \frac{\exp(-2\br\bd)}{2+8\br\bd} \times \Big[ \lambdaMin \|\vect(\Delta)\|^2_2   - \frac{\lambdaMin}{2} \|\Delta\|^2_{\mathrm{F}}  \Big]\\
	& \stackrel{(d)}{=} \frac{\exp(-2\br\bd)}{2+8\br\bd} \times \frac{\lambdaMin}{2} \|\Delta\|^2_{\mathrm{F}},
	\end{align}
	where $(a)$ follows because $\|\Delta\|_{1,1} \leq \gamma(\dimOne, \dimTwo) \|\Delta\|_\mathrm{{F}}$, $(b)$ follows from \eqref{eq:population_correlation}, $(c)$ follows from the Courant-Fischer theorem (because $\Expectation[\vect(\varPhi(\rvbx))\vect(\varPhi(\rvbx))^T]$ is a symmetric matrix) and Assumption \ref{lambdamin}, and $(d)$ follows because $\|\vect(\Delta)\|_2 = \|\Delta\|_{\mathrm{F}}$.
\end{proof}

\section{Bounds on the tensor maximum norm of the gradient of the loss function} 
\label{appendix:bounds on the gradient of the GISMe}
In this Appendix, we show that, with enough samples, the entry-wise maximum norm of the gradient of the loss function evaluated at the true natural parameter is bounded with high probability. This result allows us to prove Theorem \ref{thm:finite_sample} in Appendix \ref{appendix_proof_finite_sample}. We start by stating the main result of this Appendix (Proposition \ref{prop:gradient-concentration-GISMe}). Next, we provide a supporting Lemma wherein we show that the expected value of a random variable of interest is zero. Finally, we prove Proposition \ref{prop:gradient-concentration-GISMe}.
\begin{proposition} \label{prop:gradient-concentration-GISMe}
	Let Assumptions \ref{bounds_parameter}, \ref{bounds_statistics} and \ref{bounds_statistics_maximum} be satisfied. For any $\deltaFour \in (0,1)$, any $\epsFour > 0$, the components of the gradient of the loss function $\cL_{n}(\Theta)$\footnote[7]{Ideally, one would consider the gradient of $\cL_{n}(\vect({\Theta}))$. However, for the ease of the exposition we abuse the terminology.} evaluated at $\ThetaStar$ are bounded from above as 
	\begin{align}
	\| \nabla \cL_{n}(\ThetaStar) \|_{\max} \leq \epsFour,
	\end{align}
	with probability at least $1 - \deltaFour$ as long as
	\begin{align}
	n > \frac{8 \phiMax^2\exp(4\br\bd)}{\epsFour^2}\log\Big(\frac{2\dimOne \dimTwo}{\deltaFour}\Big).
	\end{align}
\end{proposition}

\subsection{Supporting Lemma for Proposition \ref{prop:gradient-concentration-GISMe}}
\begin{lemma}\label{lemma:zero_expectation}
	For any $u \in [\dimOne]$ and $v \in [ \dimTwo]$, define the random variable
	\begin{align}
	\rvx_{uv}  = - \varPhi_{uv}(\svbx) \exp\big( -\big\langle \big\langle \ThetaStar, \varPhi(\svbx) \big\rangle \big\rangle \big). \label{xij_definition}
	\end{align}
	We have 
	\begin{align}
	\Expectation[\rvx_{uv}] = 0,
	\end{align}
	where the expectation is with respect to $\DensityXTrue$.
\end{lemma}
\begin{proof}[Proof of Lemma \ref{lemma:zero_expectation}]
	Fix any $u \in [\dimOne]$ and $v \in [ \dimTwo]$. Using \eqref{xij_definition}, we have
	\begin{align}
	\Expectation[\rvx_{uv}] & = -\int_{\svbx \in \cX} \hspace{-2mm}\DensityXTrue \varPhi_{uv}(\svbx) \exp\big( -\big\langle \big\langle \ThetaStar, \varPhi(\svbx) \big\rangle \big\rangle \big) d\svbx \stackrel{(a)}{=} \frac{-\int_{\svbx \in \cX}  \varPhi_{uv}(\svbx)  d\svbx}{\int_{\svby \in \cX} \exp\big( \big\langle \big\langle \ThetaStar, \varPhi(\svby) \big\rangle \big\rangle \big) d\svby} \\
	& \stackrel{(b)}{=} 0,
	\end{align}
	where $(a)$ follows from the definition of $\DensityXTrue$ and because $\Expectation_{\Uniform} [\Phi(\rvbx)]$ is a constant, and $(b)$ follows because $\int_{\svbx \in \cX} \varPhi(\svbx)  d\svbx = 0$ from Definition \ref{def:css}.
\end{proof}

\subsection{Proof of Proposition \ref{prop:gradient-concentration-GISMe}}
\begin{proof}[Proof of Proposition \ref{prop:gradient-concentration-GISMe}]
	Fix $u \in [\dimOne]$ and $v \in [ \dimTwo]$. We start by simplifying the gradient of the $\cL_n(\Theta)$ evaluated at $\ThetaStar$. The component of the gradient of $\cL_{n}(\Theta)$ corresponding to $\Theta_{uv}$
	evaluated at $\ThetaStar$ is given by
	\begin{align}
	\frac{\partial \cL_{n}(\ThetaStar)}{\partial \Theta_{uv}}  = - \frac{1}{n} \sum_{t = 1}^{n} \varPhi_{uv}(\svbx^{(t)}) \exp\Big( -\Big\langle \Big\langle \ThetaStar, \varPhi(\svbx^{(t)}) \Big\rangle \Big\rangle \Big).  \label{eq:gradient_GISMe_ThetaStar}
	\end{align}
	Each term in the above summation is distributed as per the random variable $\rvx_{uv}$ (see \eqref{xij_definition}). The random variable $\rvx_{uv}$ has zero
	mean (see Lemma \ref{lemma:zero_expectation}) and satisfies $
	|\rvx_{uv}| \leq 2\phiMax  \exp(2\br\bd)$ (from Assumption \ref{bounds_statistics_maximum} and arguments similar to \eqref{eq:bound_inner_product}). Using the Hoeffding's inequality, we have
	\begin{align}
	\Prob \left( \bigg|\frac{\partial \cL_{n}(\ThetaStar)}{\partial \Theta_{uv}}\bigg| > \epsFour \right) < 2\exp \left(-\frac{n \epsFour^2}{8\phiMax^2 \exp(4\br\bd)}\right). \label{eq:Hoeffding}
	\end{align}
	The proof follows by using \eqref{eq:Hoeffding} and the union bound over all $u \in [\dimOne]$ and $v \in [ \dimTwo]$.
\end{proof}

\section{Proof of Theorem 4.3}
\label{appendix_proof_finite_sample}
	In this Appendix, we prove Theorem \ref{thm:finite_sample} by invoking  \citet[Corollary 1]{negahban2012unified}. In particular, we show that strong regularization and restricted strong convexity of the loss function enables us to show that the minimizer $\hThetan$ is close to the true parameter $\ThetaStar$ in Frobenius norm. We restate the Theorem below and then provide the proof.
	\theoremfinite*
	\begin{proof}[Proof of Theorem \ref{thm:finite_sample}]
		We start by recalling Corollary 1 of \cite{negahban2012unified}.
		
		\citet[Corollary 1]{negahban2012unified}: Let $z_1, \cdots, z_n$ be i.i.d. samples of a random variable $\rvz$. Let $q(\rvz ; \theta)$ be some convex and differentiable function of $\rvz$ parameterized by $\theta \in \Upsilon$. Define
		\begin{align}
		\hthetan \in \argmin_{\theta} \frac{1}{n} \sum_{i = 1}^{n} q(z_i ; \theta) + \lambda_n \cR(\theta), \label{eq:m-est-negah}
		\end{align} 
		where $\lambda_n$ is a regularization penalty and $\cR$ is a norm. Let $\theta^*$ be the true underlying parameter, i.e., $\theta^* \in \argmin_{\theta} \Expectation[q(z; \theta)]$. Let the following be true.
		\begin{enumerate}[leftmargin=6mm, itemsep=-0.5mm]
			\item[(a)] The regularization penalty is such that $\lambda_n \geq 2\cR^*(\nabla \cL_{n}(\ThetaStar))$ where $\cR^*$ is the dual norm of $\cR$ and
			\item[(b)] The loss function satisfies a restricted strong convexity condition with curvature $\kappa > 0$, i.e., $\delta \cL_{n}(\Delta, \theta^*) \geq \kappa \|\Delta\|^2_2$ where $\delta \cL_{n}(\Delta, \theta^*)$ is the residual of the first-order Taylor expansion.
		\end{enumerate}
		Then, $\hthetan$ is such that $\| \hthetan - \theta^*\|_2 \leq 3\frac{\lambda_n}{\kappa} \Psi(\Upsilon)$ where $\Psi(\Upsilon) = \max_{\svbv \in \Upsilon \setminus \{\boldsymbol{0}\}} \frac{\cR(\svbv)}{\|\svbv\|_2}$.\\
		
		We let $z \coloneqq \rvbx$, $\theta \coloneqq \Theta$, $\hthetan \coloneqq \hThetan$, $\theta^* \coloneqq \ThetaStar$, $\Upsilon = \parameterSet$, and $q(z; \theta) \coloneqq \exp\big( -\big\langle \big\langle \Theta, \varPhi(\svbx) \big\rangle \big\rangle \big)$. While Corollary 1 of \cite{negahban2012unified} is a deterministic result, we use probabilistic analysis to show that the neccessary conditions hold resulting in a high probability bound. Now, let the number of samples satisfy
		\begin{align}
		n & \geq \max  \bigg\{\frac{128 \phiMax^4 \gamma^4(\dimOne, \dimTwo) }{\lambdaMin^2}\log\Big(\frac{4 \dimOne^2  \dimTwo^2 }{\delta}\Big), \\
		& \qquad \qquad \frac{8 \phiMax^2  (24 + 96\br\bd)^2  \exp(8\br\bd ) g^2(\dimOne, \dimTwo) \Psi^2(\dimOne, \dimTwo)}{\alpha^2\lambdaMin^2}\log\Big(\frac{4\dimOne \dimTwo}{\delta}\Big)\bigg\}  \label{eq:sample_comp_dependence}\\
		& = \Omega\bigg(\frac{ \max\big\{  \gamma^4(\dimOne, \dimTwo), g^2(\dimOne, \dimTwo) \Psi^2(\dimOne, \dimTwo) \big\} }{\alpha^2 \lambdaMin^2}\log\Big(\frac{4\dimOne^2 \dimTwo^2}{\delta}\Big) \bigg).
		\end{align}

		To choose the regularization penalty, we bound $\cR^*(\nabla \cL_{n}(\ThetaStar))$. From \eqref{eq_defn_thm} in Section \ref{sec:main results}, we have $\cR^*(\nabla \cL_{n}(\ThetaStar)) \leq  g(\dimOne, \dimTwo) \| \nabla \cL_{n}(\ThetaStar) \|_{\max}$. Then, using Proposition \ref{prop:gradient-concentration-GISMe} with $\deltaFour \mapsfrom \frac{\delta}{2}$, we have the following with probability at least $1-\frac{\delta}{2}$,
		\begin{align}
		\cR^*(\nabla \cL_{n}(\ThetaStar)) \leq g(\dimOne, \dimTwo) \epsilon \qquad \text{whenever} \qquad n > \frac{8 \phiMax^2\exp(4\br\bd)}{\epsFour^2}\log\Big(\frac{4\dimOne \dimTwo}{\delta}\Big).
		\end{align}
		As a result, we choose $\lambda_n = 2g(\dimOne, \dimTwo) \epsilon$.
		
		Now, let $\Delta = \hThetan - \ThetaStar$. Then, \citet[Lemma 1]{negahban2012unified} implies that $\Delta$ is such that $\cR(\Delta) \leq 4 \cR(\ThetaStar)$, i.e., $\Delta \in 4 \parameterSet$. Then, using Proposition \ref{prop:rsc_GISMe} with $\deltaThree \mapsfrom \frac{\delta}{2}$, we have the following with probability at least $1-\frac{\delta}{2}$,
		\begin{align}
		\delta \cL_{n}(\Delta, \ThetaStar) \geq \frac{\lambdaMin \exp(-2\br\bd )}{4 + 16\br\bd} \|\Delta\|^2_{\mathrm{F}} \qquad \text{whenever} \qquad n > \frac{128 \phiMax^4 \gamma^4(\dimOne, \dimTwo) }{\lambdaMin^2}\log\Big(\frac{4 \dimOne^2  \dimTwo^2 }{\deltaThree}\Big).
		\end{align}

		Putting everything together, we have the following with proability at least $1-\delta$,
		\begin{align}
		\|\hThetan - \ThetaStar\|_{\mathrm{F}}  & \leq \frac{(24 + 96rd) \exp(2\br\bd ) g(\dimOne, \dimTwo)\epsilon}{\lambdaMin} \Psi(\dimOne, \dimTwo).
		\end{align}
		Choosing $ \epsFour = \frac{\alpha \lambdaMin}{(24 + 96rd) \exp(2\br\bd ) g(\dimOne, \dimTwo) \Psi(\dimOne, \dimTwo)}$ completes the proof.
	\end{proof}

\section{Examples}
\label{appendix:examples}
In this Appendix, we provide more discussion on the examples of natural parameters and statistics from Section \ref{subsec:examples}.

\subsection{Polynomial natural statistic}
\label{appendix:poly}
Suppose the natural statistics are polynomials of $\rvbx$ with maximum degree $l$, i.e., $\prod_{i \in [p]} x_i^{l_i}$ such that $l_i \in [l] \cup \{0\}$ for all $ i \in [p]$ and $\sum_{i \in [p]} l_i \leq l$ for some $l < p$. 
\begin{itemize}[leftmargin=6mm, itemsep=-0.5mm]
	 \item Let $\cX = [-b,b]^p$ for $b \in \Reals$. First, we show that $\phiMax = \max\{1,b^l\}$. We have
\begin{align}
\| \Phi(\svbx) \|_{\max}  = \max_{u \in [\dimOne], v \in [ \dimTwo]} |\Phi_{uv}(\svbx)|
\leq \max\{1,b^l\}.
\end{align}
	\item Suppose $\ThetaStar$ has a bounded maximum norm $\|\Theta\|_{\max} \leq r$. The dual norm of the matrix maximum norm is the matrix $L_{1,1}$ norm. Suppose $\cX = \cB_1(b) \subset [-b,b]^p$ for $b \in \Reals$.
	Then,	
	\begin{align}
	\cR^*(\Phi(\svbx)) = \| \Phi(\svbx) \|_{1,1} \leq \sum_{\substack{l_i \in [l] \cup {0}:\\ \sum_{i} l_i \leq l}} \prod_{i \in [p]} |x_i|^{l_i} \leq (1 + \|\svbx\|_{1})^l \stackrel{(a)}{\leq} (1 + b)^l, \label{eq_bounded_l11_ss}
	\end{align}
	where $(a)$ follows because $\svbx \in \cB_1(b)$.
	\item Suppose $\ThetaStar$ has bounded Frobenius norm $\|\Theta\|_{\mathrm{F}} \leq r$. The dual norm of the Frobenius norm is the Frobenius norm itself. Suppose $\cX = \cB_1(b) \subset [-b,b]^p$ for $b \in \Reals$. Then,
	\begin{align}
	\cR^*(\Phi(\svbx)) = \| \Phi(\svbx) \|_{\mathrm{F}} \stackrel{(a)}{\leq} \| \Phi(\svbx) \|_{1,1} \stackrel{(b)}{\leq} (1 + b)^l,
	\end{align}
	where $(a)$ follows because Frobenius norm is bounded by matrix $L_{1,1}$ norm and $(b)$ follows from \eqref{eq_bounded_l11_ss}.
	\item Suppose $\ThetaStar$ has a bounded nuclear norm $\|\Theta\|_{\star} \leq r$. The dual norm of the matrix nuclear norm is the matrix spectral norm. Suppose $l = 2$ and $\cX = \cB_1(b) \subset [-b,b]^p$ for $b \in \Reals$. 
	Then, we can write $\Phi(\svbx)  = \tilde{\rvx}\tilde{\rvx}^T$ where $\tilde{\rvx} = (1,\rvx_1,\cdots, \rvx_p)$. As a result, $\Phi(\svbx)$ is a rank-1 matrix and the spectral norm is equal to the sum of the diagonal entries. Then, we have
	\begin{align}
	\cR^*(\Phi(\svbx)) = \| \Phi(\svbx) \| \leq (1 + \|\svbx\|_2^2) \stackrel{(a)}{\leq} (1+b^2),  
	\end{align}
	where $(a)$ follows because $\svbx \in \cB_2(b)$.
\end{itemize}
\subsection{Trigonometric natural statistic}
\label{appendix:sinu}
Suppose the natural statistics are sines and cosines of $\rvbx$ with $l$ different frequencies, i.e., $\sin(\sum_{i \in [p]}l_ix_i)$ $\cup$ $\cos(\sum_{i \in [p]}l_ix_i)$ such that $l_i \in [l] \cup \{0\}$ for all $i \in [p]$.
\begin{itemize}[leftmargin=6mm, itemsep=-0.5mm]
	\item Let $\cX \subset \Reals^{p}$. First, we show that $\phiMax = 1$. For any $\svbx \in \cX$, we have
\begin{align}
\| \Phi(\svbx) \|_{\max}  = \max_{u \in [\dimOne], v \in [ \dimTwo]} |\Phi_{uv}(\svbx)|
 \leq 1.
\end{align}
\item Suppose $\ThetaStar$ has bounded $L_{1,1}$ norm, i.e., $\|\Theta\|_{1,1} \leq r_1$. The dual norm of the matrix $L_{1,1}$ norm is the matrix maximum norm. Then, for any $\svbx \in \cX \subset \Reals^{p}$,
\begin{align}
\cR^*(\Phi(\svbx)) = \| \Phi(\svbx) \|_{\max} \leq \phiMax = 1.
\end{align}
\end{itemize}
\subsection{Combinations of polynomial and trigonometric statistics}
Suppose the natural statistics are combinations of polynomials of $\rvbx$ with maximum degree $l$, i.e., $\prod_{i \in [p]} x_i^{l_i}$ such that $l_i \in [l] \cup \{0\}$ for all $ i \in [p]$ and $\sum_{i \in [p]} l_i \leq l$ for some $l < p$
 as well as  sines and cosines of $\rvbx$ with $\tilde{l}$ different frequencies, i.e., $\sin(\sum_{i \in [p]}\tilde{l}_ix_i)$ $\cup$ $\cos(\sum_{i \in [p]}\tilde{l}_ix_i)$ such that $\tilde{l}_i \in [\tilde{l}] \cup \{0\}$ for all $i \in [p]$.
\begin{itemize}[leftmargin=6mm, itemsep=-0.5mm]
	\item Let $\cX = [-b,b]^p$ for $b \in \Reals$. From Appendix \ref{appendix:poly} and Appendix \ref{appendix:sinu}, it is easy to verify that $\phiMax = \max\{1,b^l\}$. 
	\item Now, suppose $\ThetaStar$ has bounded $L_{1,1}$ norm, i.e., $\|\Theta\|_{1,1} \leq r_1$. The dual norm of the matrix $L_{1,1}$ norm is the matrix maximum norm. Then, for any $\svbx \in \cX = [-b,b]^p$,
	\begin{align}
	\cR^*(\Phi(\svbx)) = \| \Phi(\svbx) \|_{\max} \leq \phiMax = \max\{1,b^l\}.
	\end{align}
\end{itemize}

\section{Bound on $g(\dimOne, \dimTwo)$}
\label{appendix:dual_norm}
In this Appendix, we provide a bound on $g(\dimOne, \dimTwo)$ defined in \eqref{eq_defn_thm} in Section \ref{sec:main results} for the entry-wise $L_{p,q}$ norms $(p,q \geq 1)$, the Schatten $p$-norms $(p \geq 1)$, and the operator $p$-norms $(p \geq 1)$.
\subsection{The entry-wise $L_{p,q}$ norms}
Let $\widetilde{\cR}(\cdot)$ denote the entry-wise $L_{p,q}$ norm for some $p,q \geq 1$. We show that for any matrix $\bM \in \Reals^{\dimOne\times \dimTwo}$, $g(\dimOne, \dimTwo) = \dimOne^{\frac{1}{p}}  \dimTwo^{\frac{1}{q}}$, i.e., 
\begin{align}
\widetilde{\cR}(\bM) \leq \|\bM\|_{\max} \times \dimOne^{\frac{1}{p}}  \dimTwo^{\frac{1}{q}}.
\end{align}
 By the definition of the entry-wise $L_{p,q}$ norm, we have
\begin{align}
\widetilde{\cR}(\bM) = \bigg(\sum_{j \in  [\dimTwo]} \bigg(\sum_{i \in [\dimOne]} | M_{ij} |^p \bigg)^{\frac{q}{p}} \bigg)^{\frac{1}{q}} & \leq \bigg(\sum_{j \in  [\dimTwo]} \bigg(\sum_{i \in [\dimOne]} \|\bM\|_{\max}^p \bigg)^{\frac{q}{p}} \bigg)^{\frac{1}{q}}  = \dimOne^{\frac{1}{p}}  \dimTwo^{\frac{1}{q}} \|\bM\|_{\max}.
\end{align}

\subsection{The Schatten $p$-norms}
Let $\widetilde{\cR}(\cdot)$ denote the Schatten $p$-norm for some $p \geq 1$. We show that for any matrix $\bM \in \Reals^{\dimOne\times \dimTwo}$, $g(\dimOne, \dimTwo) =\sqrt{\min\{\dimOne, \dimTwo\}\dimOne \dimTwo}$, i.e., 
\begin{align}
\widetilde{\cR}(\bM) \leq \|\bM\|_{\max} \times\sqrt{\min\{\dimOne, \dimTwo\}\dimOne \dimTwo}.
\end{align}
Let the rank of $\bM$ be denoted by $r$ and the singular values of $\bM$ be denoted by $\sigma_i(\bM)$ for $i \in [r]$.
By the definition of the Schatten $p$-norm, we have
\begin{align}
\widetilde{\cR}(\bM) = \bigg(\sum_{i \in [r]} \sigma_i^p(\bM) \bigg)^{\frac{1}{p}} \stackrel{(a)}{\leq} \sum_{i \in [r]} \sigma_i(\bM) & \stackrel{(b)}{\leq} \sqrt{r\dimOne \dimTwo} \|\bM\|_{\max} \stackrel{(c)}{\leq} \sqrt{\min\{\dimOne, \dimTwo\}\dimOne \dimTwo} \|\bM\|_{\max}, 
\end{align}
where $(a)$ follows because of the monotonicity of the Schatten $p$-norms, $(b)$ follows because $\|\bM\|_{\star} \leq \sqrt{r\dimOne \dimTwo} \|\bM\|_{\max}$, and $(c)$ follows because $r \leq \min\{\dimOne, \dimTwo\}$.

\subsection{The operator $p$-norms}
Let $\widetilde{\cR}(\cdot)$ denote the operator $p$-norm for some $p \geq 1$. We show that for any matrix $\bM \in \Reals^{\dimOne\times \dimTwo}$, $g(\dimOne, \dimTwo) =\dimOne^{\frac{1}{p}}  \dimTwo^{1-\frac{1}{p}}$, i.e., 
\begin{align}
\widetilde{\cR}(\bM) \leq \|\bM\|_{\max} \times \dimOne^{\frac{1}{p}}  \dimTwo^{1-\frac{1}{p}}.
\end{align}
Let $q = \frac{p}{p-1}$. For $i \in \dimOne$, let $[\bM]_{i}$ denote the $i^{th}$ row of $\bM$. By the definition of the operator $p$-norm, we have
\begin{align}
\widetilde{\cR}(\bM) = \max_{\svby : \| \svby\|_p = 1} \|\bM \svby \|_{p} \stackrel{(a)}{\leq} \dimOne^{\frac{1}{p}} \max_{\svby : \| \svby\|_p = 1}  \|\bM \svby \|_{\infty} & \stackrel{(b)}{\leq} \dimOne^{\frac{1}{p}} \max_{\svby : \| \svby\|_p = 1}  \max_{i \in [\dimOne]} \|[\bM]_{i}\|_{q} \|\svby\|_{p} \\
& \leq \dimOne^{\frac{1}{p}}   \max_{i \in [\dimOne]} \|[\bM]_{i}\|_{q} \\
& \stackrel{(c)}{\leq} \dimOne^{\frac{1}{p}}  \dimTwo^{\frac{1}{q}}  \max_{i \in [\dimOne]} \|[\bM]_{i}\|_{\infty} = \dimOne^{\frac{1}{p}}  \dimTwo^{1-\frac{1}{p}} \|\bM\|_{\max},
\end{align}
where $(a)$ follows because $\|\svbv\|_p \leq m^{\frac{1}{p}} \|\svbv\|_{\infty}$ for any vector $\svbv \in \Reals^{m}$ and $p \geq  1$, $(b)$ follows from the definition of the infinity norm of a vector and using the H\"{o}lder's inequality, and $(c)$ follows because $\|\svbv\|_q \leq m^{\frac{1}{q}} \|\svbv\|_{\infty}$ for any vector $\svbv \in \Reals^{m}$ and $q \geq  1$.

\end{document}